%% file: bare_jrnl.tex
\documentclass[9pt,journal]{IEEEtran}
\usepackage{times}
\usepackage{epsfig}
\usepackage{graphicx}
\usepackage{amsmath}
\usepackage{amssymb}
\usepackage{amsthm}
\usepackage[inline]{enumitem}
\usepackage{algorithm}
\usepackage[noend]{algpseudocode}
\usepackage{algpseudocode}
\usepackage{array}
\usepackage{multirow}
\usepackage{subcaption}
\usepackage{algcompatible}
\usepackage{bbm}
\usepackage{dsfont}

\usepackage{balance}

\usepackage[hidelinks,colorlinks,pagebackref,breaklinks=true,bookmarks=true]{hyperref}

\ifCLASSINFOpdf
\else
\fi

\begin{document}
\newcommand*\rfrac[2]{{}^{#1}\!/_{#2}}
%
\title{Approximated Robust Principal Component Analysis for Improved General Scene Background Subtraction}
%
%
%

\author{Salehe~Erfanian~Ebadi,~\IEEEmembership{Student~Member,~IEEE,}
        Valia~Guerra~Ones,
        and~Ebroul~Izquierdo,~\IEEEmembership{Senior Member,~IEEE}
}



%



\maketitle

\begin{abstract}
\input{IEEE_00_Abstract}
\end{abstract}


%
\IEEEpeerreviewmaketitle

\section{Introduction}
\input{IEEE_01_Introduction}

\subsection{Review of Related Work}
\label{sec:RelatedMotiv}
\input{IEEE_02_RelatedMotiv}

\section{Fundamentals}
\label{sec:Methodology}
\input{IEEE_03_Methodology}

\section{Matrix Decomposition for Video Analysis}
\label{sec:Proposal}

\subsection{$\tau$-Decomposition, An Approximated RPCA for Handling Non-Static Backgrounds} \label{subsec:GenMovGoDec}
\input{IEEE_03_Methodology_1}

\subsection{Foreground Detection with $\ell_{2,1}$-Norm Block-Sparsity}
\label{subsec:BlockSparse}
\input{IEEE_03_Methodology_2}

\subsection{Ghost-Removal: Statistical Leverage Scores for Nearest Estimation of the Low-Rank and the Sparse Components}
\label{subsec:GhostRemove}
\input{IEEE_03_Methodology_3}

\subsection{Special Case: SVD-Free Algorithm for Fast Decomposition}
\label{subsec:SVDFree}
\input{IEEE_03_Methodology_4}

\section{Experiments and Results}
\label{sec:ExperimentsRes}
\input{IEEE_04_ExperimentsResults}

\subsection{Evaluating the Ghost-Removal Algorithm}
\input{IEEE_04_ExperimentsResults_1}

\subsection{Comparison of the RPCA-LBD and our Proposed Block-Sparse Algorithm}
\input{IEEE_04_ExperimentsResults_2}

\subsection{Decomposition and Reconstruction of Video Sequences with SVD-Free Algorithm}
\input{IEEE_04_ExperimentsResults_3}

\subsection{Foreground Segmentation Evaluation}
\label{subsec:FGsegEval}
\input{IEEE_04_ExperimentsResults_4}

\section{Discussion and Future Work} \label{sec:conc}
\input{IEEE_05_DiscussFuture}

\ifCLASSOPTIONcaptionsoff
  \newpage
\fi




\bibliographystyle{IEEEtran}
\bibliography{mybib}
\balance

\vfill

%
%
%

%

\begin{IEEEbiography}[{\includegraphics[width=1in,height=1.25in,clip,keepaspectratio]{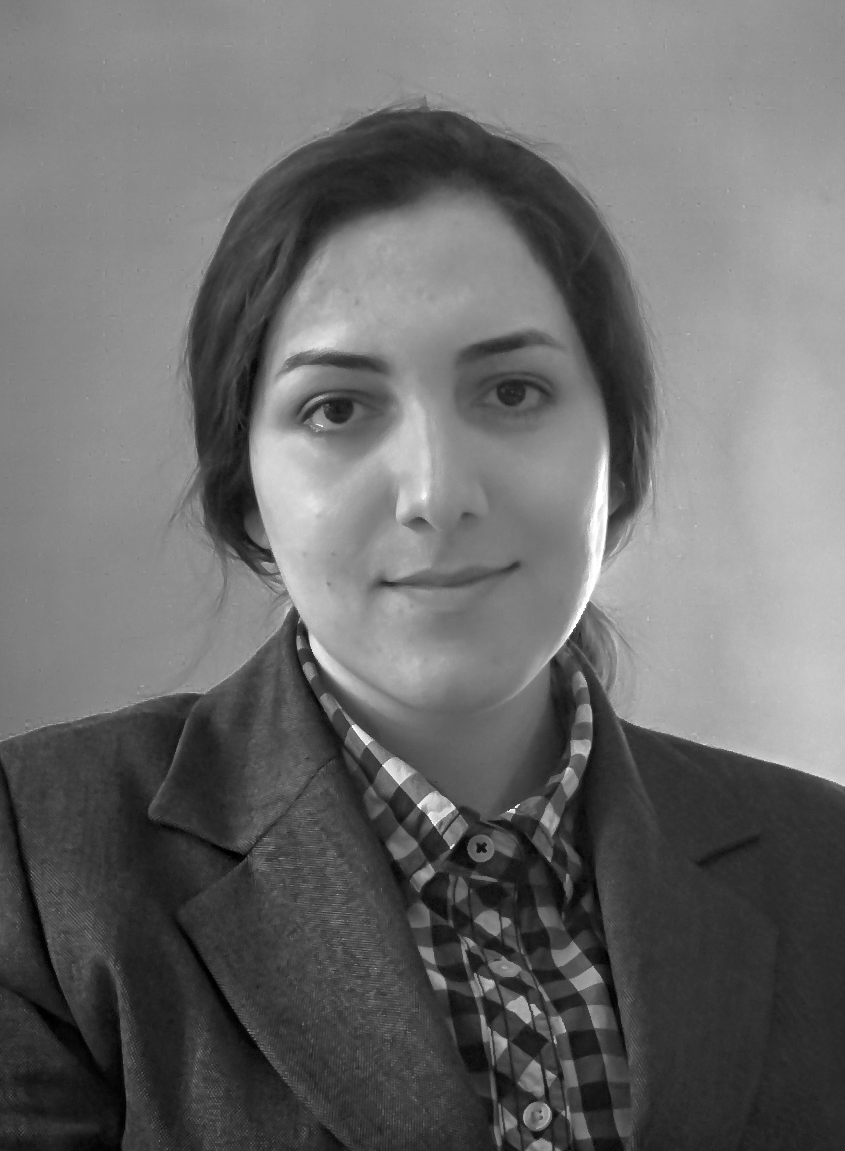}}]{Salehe Erfanian Ebadi}
received her B.Sc. degree in Telecommunications Electrical Engineering from Sadjad University of Technology, Mashhad, Iran in 2011, and her M.Sc. degree in Digital Signal Processing from Queen Mary University of London, London, U.K. in 2012. She is currently pursuing a Ph.D. degree in Electronic Engineering at the Multimedia and Vision Group, School of Electronic Engineering and Computer Science, Queen Mary University of London. Her research interests include Image and Video Processing, Matrix Decomposition Techniques for Computer Vision, and Machine Learning.
\end{IEEEbiography}

\begin{IEEEbiography}[{\includegraphics[width=1in,height=1.25in,clip,keepaspectratio]{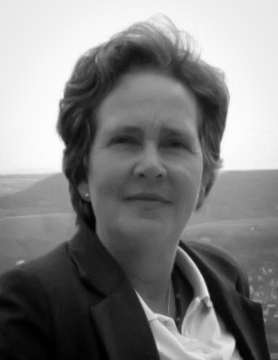}}]{Valia Guerra Ones}
Valia Guerra Ones received the M.S. and Ph.D. degrees in applied mathematics from the University of Havana,  Havana, Cuba, in 1995 and 1998, respectively. Currently, she is with the Department of Electronic Engineering, Queen Mary University of London, London, U.K., and she is visiting professor in the Department of Applied Mathematics, Delft University of Technology, Delft, The Netherlands.
Her research interest include numerical linear algebra and its applications in Image Processing, randomized matrix algorithms and ill-posed problems, and regularization methods.
\end{IEEEbiography}


\begin{IEEEbiography}[{\includegraphics[width=1in,height=1.25in,clip,keepaspectratio]{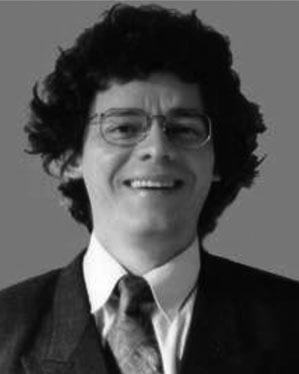}}]{Ebroul Izquierdo}
(SM'03) received the M.Sc. degree in 1988 from Berlin, Germany, the C.Eng. degree in 1999 from London, U.K., and the Dr. Rerum Naturalium (Ph.D.) from the Humboldt University, Berlin, Germany, in the field of numerical approximation of algebraic-differential equations. He is a Chair of Multimedia and Computer Vision and the Head of the Multimedia and Vision Group in the School of Electronic Engineering and Computer Science at Queen Mary, University of London, London, U.K. He was a Senior Researcher at the Heinrich-Hertz Institute for Communication Technology, Berlin, Germany, and the Department of Electronic Systems Engineering of the University of Essex. He holds several patents in the area of multimedia signal processing and has published more than 500 technical papers including chapters in books. Prof. Izquierdo is a Chartered Engineer, a Fellow of The Institution of Engineering and Technology (IET), a member of the British Machine Vision Association, Past Chairman of the IET professional network on Information Engineering, a member of the Visual Signal Processing and Communication Technical Committee of the IEEE Circuits and Systems Society and a member of the Multimedia Signal Processing technical committee of the IEEE. He is or has been associated and Guest Editor of several relevant journals in the field including the IEEE TRANSACTIONS ON CIRCUITS AND SYSTEMS FOR VIDEO TECHNOLOGY, the EURASIP Journal on Image and Video processing, the Elsevier Journal Signal Processing: Image Communication, The EURASIP Journal on Applied Signal Processing, the IEEE Proceedings on Vision, Image and Signal Processing, the Journal of Multimedia Tools and Applications and the Journal of Multimedia. He has been a member of the organizing committee of several conferences and workshops in the field and has chaired special sessions and workshops in International Conference on Image Processing, International Conference on Acoustics, Speech, and Signal Processing, and International Symposium on Circuits and Systems.
\end{IEEEbiography}




\end{document}

%% file: IEEE_00_Abstract.tex
The research reported in this paper addresses the fundamental task of separation of locally moving or deforming image areas from a static or globally moving background. It builds on the latest developments in the field of robust principal component analysis, specifically, the recently reported practical solutions for the long-standing problem of recovering the low-rank and sparse parts of a large matrix made up of the sum of these two components. This article addresses a few critical issues including: embedding global motion parameters in the matrix decomposition model, i.e., estimation of global motion parameters simultaneously with the foreground/background separation task; considering matrix block-sparsity rather than generic matrix sparsity as natural feature in video processing applications; attenuating background ghosting effects when foreground is subtracted; and more critically providing an extremely efficient algorithm to solve the low-rank/sparse matrix decomposition task. The first aspect is important for background/foreground separation in generic video sequences where the background usually obeys global displacements originated by the camera motion in the capturing process. The second aspect exploits the fact that in video processing applications the sparse matrix has a very particular structure, where the non-zero matrix entries are not randomly distributed but they build small blocks within the sparse matrix. The next feature of the proposed approach addresses removal of ghosting effects originated from foreground silhouettes and the lack of information in the occluded background regions of the image. Finally, the proposed model also tackles algorithmic complexity by introducing an extremely efficient ``SVD-free'' technique that can be applied in most background/foreground separation tasks for conventional video processing.

%% file: IEEE_01_Introduction.tex
\IEEEPARstart{T}{he separation} of locally moving or deforming image areas from a static or globally moving background is a basic video processing task with manifold applications including automated anomaly detection in video surveillance, face alignment for recognition and authentication, human motion analysis, action recognition, object tracking, video summarization retrieval and editing, and object based video coding. A plethora of algorithms and techniques to achieve this task has been developed since the early days of digital image processing with varied degrees of performance and complexity.
The research presented in this paper also addresses this fundamental task by leveraging and building on recent developments in the field of \textit{Robust Principal Component Analysis} (RPCA). Specifically, the work reported here has been inspired by the critical breakthrough accomplished by Cand\`{e}s \textit{et al.}~\cite{Candes11}, where the authors provided a practical and feasible solution for the long-standing problem of recovering the low-rank and sparse parts of a large matrix made up of the sum of these two components.
In other words, assuming that a given matrix $A$ consists of the sum of a low-rank and a sparse component, there exists a feasible solution, by which the exact recovery of the original matrix from its decomposed parts becomes possible~\cite{Candes11}.
In the particular context of video processing, the $2$-dimensional matrix $A$ stores pixel information of a video sequence or a set of images by concatenating each frame or image as columns in $A$. Then, the background part of the video sequence is modeled by the low-rank matrix, while the locally deforming parts constitute the sparse matrix component. Cand\`{e}s \textit{et al.}~\cite{Candes11} showed that under certain assumptions, the low-rank/sparse decomposition problem can be solved by means of a convex optimization. In that seminal paper the authors referred to their proposed approach as \textit{Principal Component Pursuit} (PCP).

PCP has led to impressive results in background modeling, foreground detection, removal of shadows and specularities in images, and face alignment for recognition. More importantly, it shows a high potential for improving the performance of RPCA based solutions, motivating the work presented in this paper. An approximated RPCA method for general scene background subtraction is proposed here, that endeavors to overcome some of the limitations of related algorithms and techniques reported over the last few years. In particular, this approach addresses a number of critical issues of RPCA including the following four main aspects: embedding global motion parameters in the model, i.e., estimation of global motion parameters simultaneously with the foreground/background separation task; considering matrix block sparsity rather than generic matrix sparsity as natural feature in video processing applications; attenuating background ghosting effects when foreground is subtracted; and more critically providing an extremely efficient algorithm to solve the low-rank/sparse decomposition task. The first aspect is very important for background/foreground separation in general video sequences where the background usually obeys a global motion originated by the camera motion in the capturing process. Here the proposed model aims at also estimating the global motion parameters while performing the targeted background/foreground separation task. The second aspect exploits the fact that in video processing applications the sparse matrix has a very particular structure; i.e., the non-zero matrix entries are not randomly distributed but they build small blocks within the sparse matrix. The next feature of the proposed approach addresses removal of ghosting effects originated by foreground silhouettes and the lack of information in the occluded regions of the background in the image. Another important aspect relates to the fact that RPCA approaches are computationally expensive involving many \textit{Singular Value Decompositions} (SVD) of large matrices. The proposed model also addresses this issue by introducing an extremely efficient ``SVD-free'' technique that can be applied in most background/foreground separation tasks for conventional video processing.

The rest of this paper is organized as follows. In Section~\ref{sec:RelatedMotiv} a brief survey of related works as well as their limitations is provided. The main contributions and fundamentals of the research reported here are outlined in Section~\ref{sec:Contributions}. Section~\ref{sec:Notation} introduces notation and definitions used in this paper. The RPCA framework is presented in Section~\ref{sec:RPCAFramework}. Next, a thorough discussion on the proposed method is presented in section~\ref{sec:Proposal}. More specifically, a technique called $\tau$-Decomposition is described in~\ref{subsec:GenMovGoDec}, and~\ref{subsec:BlockSparse} discusses the proposed foreground detection algorithm assuming block sparsity. Sections~\ref{subsec:GhostRemove} and~\ref{subsec:SVDFree} introduce a strategy for removal of unwanted absorption of foreground pixels into background, i.e., ghost-attenuation, and a SVD-free algorithm for fast decomposition respectively. At the end selected results from a comprehensive evaluation of the proposed techniques is presented in Section~\ref{sec:ExperimentsRes} which are used to validate the introduced approaches. The paper concludes with few remarks and possibilities for future work.

The preliminary work of this article has appeared in~\cite{erfanian2015ICIP}.

%% file: IEEE_02_RelatedMotiv.tex
During the past decades many methods have been developed to address and solve the problems in background modeling and foreground detection.
Several surveys are also dedicated to this topic~\cite{BouwmansSurvey2011},~\cite{BouwmansSurvey2014}. Here in a short survey a few noteworthy methods are acknowledged.
A well-studied approach is optical flow estimation and object-based motion segmentation, which was addressed in an early work by~\cite{[24-57]OpticalFlow1998}. Another popular approach is the \textit{Markov Random Fields} (MRF) based methods in which the background is modeled and then subtracted using region-level and motion-based information as in~\cite{[24-33]RegionLevelMRF2005} and~\cite{[1-197]ForegroundMRF2012}. Nevertheless, these methods are heavily dependent on the motion estimation algorithm. The segmentation may become unsatisfactory when the foreground objects undergo large displacement. Moreover, if part of the object contour is obscured, the detection result around this part may not be very accurate. A work based on classification algorithms~\cite{[1-206]ASelfOrg2008} used a self-organizing neural network for learning background model motion patterns in videos from stationary cameras in surveillance applications.
The work by Wang \textit{et al.}~\cite{[4-50]Layers2006} proposed an iterative method to achieve layered motion segmentation, with each layer describing a region's motion, intensity, shape, and opacity. A method named ViBe~\cite{[6-3]ViBe2011} provided a non-parametric background segmentation with feedback and dynamic controllers that achieved higher processing speeds compared to previous works. Efforts have been made to achieve real-time background subtraction as in~\cite{[1-57]RealTime2010} where a GPU implementation with SVMs is reported.
The problem of bootstrapping in heavily cluttered videos has been addressed by a progressive model in~\cite{[1-72]PatchBased2010} based on an incremental, patch-based method for background initialization. Another popular approach the Bayesian formalism for statistical modeling of background has been studied extensively~\cite{[1-186]Statistical2004}. A recent work~\cite{[1-94]BayesianRPCA2011} built on the success of the robust PCA with introducing a Bayesian framework for classification of foreground and background pixels. However, Bayesian approaches do not fully solve the problem of absorption of a still foreground region into the background.
\textit{Independent Component Analysis} (ICA) of serialized images from a training sequence, is described in~\cite{[1-313]ICA2009} in which the resulting demixing vector is computed and compared to that of a new image in order to separate the foreground from a reference background image.
In~\cite{[9-34]SparseVariational2005} a model for separating images into texture and piecewise smooth parts, exploiting variational and sparsity mechanisms was proposed.


In 1999, Oliver \textit{et al.}~\cite{OliverEtAl2000} proposed to model the background by \textit{Principal Component Analysis} (PCA). They developed the theory of \textit{Robust Subspace Learning} (RSL) for making linear learning methods robust to outliers which are common in realistic training sets.
An issue is that PCA provides a robust model of the probability distribution function, but not of the moving objects since they do not have a significant contribution to the model. The limitations arising from this problem are firstly the size of the foreground object must be small and do not appear in the same location for a long period of time; and secondly the outliers of the foreground objects may be absorbed into the background mode without a mechanism of robust analysis.
Although there are several PCA improvements~\cite{TorreBlack03AFrameworkRSL} that addressed the limitations of classical PCA with respect to outlier and noise -- yielding the field of robust PCA -- these methods may not achieve sufficient performances for applications such as surveillance or general video processing that require fast and very accurate results where the input data might contain corruptions.
Recent advances in rank minimization \cite{Candes11},~\cite{WrightGaneshRaoPengMa09} have shown that it is indeed possible to efficiently and exactly recover low-rank matrices despite significant corruption, using tools from convex programming. Thus paving the way for the development of truly robust PCA based techniques.

In a related context, Zhou \textit{et al.}~\cite{ZhouEtAl13DECOLOR} proposed a method called DECOLOR in which they segmented moving objects from an image sequence. They contributed to the field of RPCA by formulating the problem as outlier detection and making use of low-rank modeling to deal with complex and moving background. They incorporated a Markov Random Fields framework to obtain a prior knowledge of the spatial distribution of the outliers (foreground objects). Compared with PCP, DECOLOR uses a non-convex penalty and MRFs for the optimization, which is more greedy to detect outlier regions that are relatively dense and contiguous. However, since DECOLOR minimizes a non-convex energy via alternating optimization, it converges to a local optimum with results depending on initialization of the foreground support, while PCP always minimizes its energy globally.
Recently Liu \textit{et al.}~\cite{liu2015TIP} proposed a structured sparsity based method in which a motion-saliency measurement for norm minimization has been used, to extend the Augmented Lagrange Multiplier method for solving the RPCA. This method obtains promising results in foreground detection for scaled-down low resolution video sequences captured by static cameras in specific environments.

Another important related work, this time addressing the complexity issue, was reported by Zhou and Tao~\cite{ZhouTao11GoDec}. The proposed technique aims at providing an approximate solution of the low-rank/sparse decomposition problem in a noisy case.

\subsection{Contributions of This Paper} \label{sec:Contributions}
This paper addresses some of the issues of previous research in the field of RPCA, to better adapt this algorithm to a problem of generic scene background subtraction. The main contributions of this work can be summarized as follows.
\begin{enumerate} \itemsep1pt \parskip0pt \parsep0pt
\item A novel algorithm is proposed which is capable of addressing the background separation problem for globally moving background.
Our approximated RPCA optimization problem called $\tau$-Decomposition includes parameters modeling global motion of background regions and also entails a Gaussian additive noise part.
This model's robustness to camera movement and dynamic backgrounds has been tested and demonstrated in this paper.
\item Enforcing structured (block) sparsity via a $\ell_{2,1}$-norm minimization based on the fact that in real-world scenes the non-zero elements of the sparse component generally appear in blocks. The method adds robustness to the model making it less prone to illumination changes, variations of object size, and noise while achieving enhances in overall segmentation performance.
\item An efficient initialization method for the low-rank and sparse matrices is proposed. It improves the separation of the moving objects from the background by attenuating the effect of locally moving image areas that usually leak through and remain within the background. This problem is known as ghosting effect.
The strategy involves the exploitation of statistical leverages for measuring the contribution of each column to the non-uniformity of the structure of the sparse matrix. In other words, a pseudo-motion saliency map generated from the video frames is used to improve the classification of foreground and background parts during the iterative process.
\item A SVD-free approximated RPCA algorithm is proposed for solving the optimization problem. This strategy not only delivers similarly accurate results as its counterpart using SVD, but is much faster, since the most expensive computation in RPCA-based methods is the SVD calculation. In this model the low-rank matrix is assumed to be of rank 1. That is the columns of the low-rank part are assumed to be linearly dependent up to a global transformation, naturally modeling what happens in real world sequences in which background undergoes global transformations only. Here the objective is the reconstruction of the original video sequence calculating the ``global'' frame that describes the globally moving part (background), the $n$ frames (or columns of a sparse matrix $S$) containing only the locally deforming or moving objects, and the parameter vector $\tau$ describing the global motion of the scene usually reflecting camera displacements or zooming.
\end{enumerate}

%% file: IEEE_03_Methodology.tex
\subsection{Notation} \label{sec:Notation}
Throughout this document the following notation is used.
Consider a video sequence $I_j$, $j=1,\dots,n$ of $n$ frames of size $w \times h$. For the sake of consistency  the same notation $I_j$ is used to represent both the video frame and the matrix containing it.
Let $A_j \in \mathbb{R}^{m}$ ($m = w \times h$) be the $m$-vector whose entries are the elements of matrix $I_j$. $A_j$ is produced by concatenating all elements of $I_j$ in row-order in a single column or vector containing $m$ elements. Then the matrix $A = [A_1, \dots, A_n] \in \mathbb{R}^{m \times n}$ contains $n$ frames of a video sequence with $m \gg n$. Then, $a_{ij}$ refers to a single element (pixel) $i$ in the frame $j$ of matrix $A$. Here $i=1,\dots,m$ and $j=1,\dots,n$.
Define $mat(\cdot)$ a mapping operation from the $m$-dimensional space into the $w \times h$ matrix as $\mathbb{R}^{m} \xrightarrow{} \mathbb{R}^{w \times h}$, i.e., $mat(A_j)$ is equal to video frame $I_j$.
The following matrix norms are utilized in this paper:
\begin{itemize}
\item Frobenius Norm: $\|A\|_F = \sqrt{\sum_{i,j} A_{ij}^2}$
\item $\ell_1$-norm: $\|A\|_1 = \sum_{i,j} |A_{ij}|$
\item Adapted dissymmetric norm ($\ell_{\alpha,\beta}$-norm):\\$\|A\|_{\alpha,\beta} = (\sum_j (\sum_i A_{ij}^\alpha)^{\frac{\beta}{\alpha}})^{\frac{1}{\beta}}$
\item $\ell_{2,1}$-norm: $\|A\|_{2,1} = \sum_j \|A_j\|_2$ (which is the $\ell_1$-norm of the vector formed by taking the $\ell_2$-norms of the columns of the underlying matrix)
\item Nuclear norm: $\|A\|_* = \sum_i \sigma_i (A)$ (where $\sigma_i (A)$ is the $i$-th largest singular value of $A$)
\end{itemize}

In this paper for practical reasons, we assume that the matrix $A$ is decomposable; i.e., the matrix $A$ is close to a matrix that can be written as the sum of a low-rank matrix $L$ with singular vectors that are not spiky and a sparse matrix $S$ with a uniform and random pattern of sparsity.

\subsection{RPCA Framework} \label{sec:RPCAFramework}
The work on RPCA-PCP developed in~\cite{Candes11}, and later used by~\cite{PengiGaneshWrightXu12}, exploits convex optimization to address the robust PCA problem. Under some assumptions, this approach called \textit{Principal Component Pursuit} (PCP) recovers the low-rank and the sparse matrices in synthetic and real data. Given a large data matrix $A$ the decomposition is defined as $A = L + S$ with:
\begin{equation}\label{eq:candes1}
\min \|A - L\| \quad \text{subject to} \quad rank(L) \leq k
\end{equation}
where $L$ is low-rank and $S$ is sparse. The most valid formulation in (\ref{eq:candes1}) uses the $\ell_0$-norm to minimize the energy function:
%
\begin{equation}
\operatorname*{arg\,min} Rank(L) + \lambda \|S\|_0 \quad \text{subject to} \quad A = L + S
\end{equation}

Here $\lambda$ is an arbitrary balancing parameter. The above problem is NP-hard and thus unfeasible for practical applications. To provide a feasible solution, the authors in~\cite{Candes11} proposed to minimize a surrogate model using $\lambda = \frac{1}{\sqrt{\max(m,n)}}$, and the $\ell_1$ and nuclear norms instead. This leads to the convex problem:
\begin{equation}
\operatorname*{arg\,min} \|L\|_{*} + \lambda \|S\|_1 \quad \text{subject to} \quad A = L + S
\end{equation}
Although this formulation leads to a computationally feasible solution, the complexity is still high involving the calculation of many SVDs for a very large matrix.
Existing RPCA algorithms often concentrate on finding exact and meaningful decompositions. However, their complexity is often uncontrollable due to their automatic and iterative solving procedure, which makes them unsuitable for computer vision applications.
In order to reduce complexity of the algorithm, the approximated RPCA (GoDec) was proposed~\cite{ZhouTao11GoDec}. This simplified version aims at decomposing a particular matrix into its low-rank and sparse components in the presence of noise.
The algorithm estimates the low-rank part $L$ and the sparse part $S$ of a large matrix containing an additive noise part $G$ as:
\begin{equation}
A = L + S + G
\end{equation}
GoDec alternatively assigns the low-rank approximation of $A - S$ to $L$ and the sparse approximation of $A - L$ to $S$. The authors also proved that the objective value below converges to a local minimum, while $L$ and $S$ linearly converge to local optimums.
\begin{equation}
\min_{L,S} \|A - L - S\|_F^2 \quad \text{such that} \quad rank(L) \leq k, card(S) \leq \kappa
\end{equation}

The model $A = L + S + G$, can handle approximated decomposition in more realistic situations when the exact and unique decomposition does not exist. In the optimization process the rank of $L$ and cardinality of $S$ are fixed. 
This imposes limitations to decomposition of unconstrained real-world video sequences, because usually the cardinality of $S$ varies and thus must remain flexible. 
Moreover, the hard-thresholding towards $S$ requires sorting all its entries' magnitudes and thus is computationally expensive. Later a similar method was proposed in~\cite{SSTZhou2013} by introducing a \textit{Lagrange} formulation as below:
\begin{equation}\label{eq:sstgodec}
\min_{L,S} \|A - L - S\|_{F}^2 + \lambda \|S\|_1 \quad \text{such that} \quad rank(L) \leq k
\end{equation}
The tuning \textit{Lagrangian} parameter $\lambda$, which acts as a soft-threshold value is much more convenient than determining the cardinality of $S$, because the resulting decomposition error is more robust to the change of the \textit{Lagrangian} parameter.
To solve (\ref{eq:sstgodec}) the authors decompose the sparse part as the sum of several low-rank matrices, each one corresponding to objects in the scene sharing the same motion trajectory. However this method does not handle cases with moving cameras (in which parts of the scene move uniformly with the camera motion) or cases where part of the background undergoes a global motion trajectory.

%% file: IEEE_03_Methodology_1.tex
In this section building on~\cite{Candes11}, ~\cite{PengiGaneshWrightXu12}, and ~\cite{ZhouTao11GoDec}, an approximated RPCA model for handling non-static backgrounds is proposed.
First, a transformation $\tau$ modeling potential global motion that the foreground region undergoes, is introduced into the optimization task. Basically, it is assumed that the columns of the matrix $L$ are linearly dependent up to a certain parametric transformation.
Given a data matrix $A$ whose columns are the frames of a video sequence, captured by a moving camera, the decomposition of matrix $A$ is in the following form:
\begin{equation} \label{eq:1}
A \circ \tau = L + S + G
\end{equation}
where $L$ is a low-rank matrix, $S$ is a sparse matrix, and $G$ is a matrix that contains the incomplete information and corruption by outliers in the original video sequence e.g. Gaussian noise. $A_j \circ \tau_j$ denotes the $j$-th frame after transformation parameterized by the vector $\tau_j \in \mathbb{R}^\rho$ where $\rho$ is the number of parameters fully describing the global motion model. Therefore $\rho = 4$ corresponds to similarity, $\rho = 6$ to affine, and $\rho = 8$ to projective transformation.
The $i$-th geometric transformation is comprised of a parameter vector $\tau_i$, $i = 1, \dots, n$ where
different spatial transformations 
can be considered.
We use the 2D parametric transforms to model the translation, rotation, and planar deformation of the background.
In particular, we also use an affine transformation where each parameter $\tau_i$ is a vector with six coefficients ($\rho = 6$).
Finally, we use the multi-resolution incremental refinement described in~\cite{Szeliski10Book}, to estimate these motion parameters.

Peng \textit{et al.}~\cite{PengiGaneshWrightXu12} proposed a mathematical formulation for equation (\ref{eq:1}) that guarantees a unique solution under some conditions. They presented an algorithm for calculating the matrices $L$, $S$, and $G$ and the motion model parameters $\tau_1, \dots, \tau_n$. The main limitation of this algorithm is its computation cost.
Here, a computationally-cheaper algorithm is proposed based on an approximated RPCA formulation.

Given the data matrix $A$ and the soft-thresholding parameter $\lambda$ the following convex optimization function recovers a low-rank matrix $L$, a sparse matrix $S$, and the motion parameter vector $\tau$ such that $A \circ \tau \approx L + S$:
\begin{equation} \label{eq:2}
\operatorname*{arg\,min}_{\substack{
            L,S,\tau\\
            rank(L) \leq k}} \|A \circ \tau - L - S\|_F + \lambda \|S\|_1
\end{equation}

The first summand guarantees the approximations of the decomposition (minimizing the residual) and the second favors the sparse matrix solution $S$ with many zero elements (i.e. sparse enough). The parameter $\lambda$ controls the contribution of each summand to the function to be minimized. $\lambda$ needs to be manually set depending on the problem to be solved and increases the model's flexibility and generalizability to different scenarios. The model is tested using variations of this parameter in our experiments with the \textit{Receiver Operating Characteristic} (ROC) performance evaluation.
We use the following alternating strategy minimizing the function for three parameters $L$, $S$, and $\tau$ one at a time with $t=1,2,\dots,p$ until the solution reaches convergence. Observe that the solution of this strategy leads to solving (\ref{eq:2}) by minimizing three reduced problems,
each being minimized independently from one another. This kind of iterative linearization has a long history in gradient algorithms.
Algorithm~\ref{lst:alg1} below describes the iterative process for the minimization process of the following three sub-problems.
\begin{equation} \label{eq:3}
\tau^t = \operatorname*{arg\,min}_\tau \|A \circ \tau - L^{t-1} - S^{t-1}\|_F^2
\end{equation}
\begin{equation} \label{eq:4}
L^t = \operatorname*{arg\,min}_{rank(L) \leq k} \|A \circ \tau^t - L - S^{t-1}\|_F^2
\end{equation}
\begin{equation} \label{eq:5}
S^t = \operatorname*{arg\,min}_S \|A \circ \tau^t - L^t - S\|_F^2 + \lambda \|S\|_1
\end{equation}

\newtheorem*{subp1}{Solving the first optimization problem -- equation (\ref{eq:3})}
\begin{subp1} \label{subp1:FirstSubProb} \normalfont
In the first optimization problem the parameters $\tau_i$, $i=1, \dots, n$ transform each of the columns of $A$ individually. Therefore $n$ minimization problems must be solved. The $i$-th problem consists of inferring the transformation parameters that transform the $i$-th frame $A_i$ to the image corresponding to the matrix $L_i^{t-1} - S_i^{t-1}$. This problem can be written as a weighted least squares minimization where the solutions $\tau_i$ have a closed-form.
It is well known that the solution of minimum norm of a least squares problem is unique and has a simple closed-form solution in terms of the Singular Value Decomposition of the system matrix. However, the obtained solution is highly unstable with respect to small changes in the data because the rows of the system matrix corresponding to nearby pixels in adjacent frames tend to be quite similar and the system matrix is generally ill-conditioned. There are different strategies for regularizing the solution of an ill-conditioned linear system. Here, the incremental refinement described in~\cite{PengiGaneshWrightXu12} is used, where a local linearization is applied on $A_i \circ \tau_i$ as function of the parameters. The main idea is using an initial approximation for the parameters $\tau_i$, $i = 1, \dots, n$ that is iteratively improved applying a refinement process and assuming a linear local behavior of $A_i \circ \tau_i$ around the initial approximation. For $t = 1$ the initial approximation of the parameters is obtained using the robust multiresolution method for prealignment described in~\cite{odobez1995robust}.
\end{subp1}
\newtheorem*{subp2}{Solving the second optimization problem -- equation (\ref{eq:4})}
\begin{subp2}  \normalfont
Next step is to calculate the rank-$k$ matrix that is the nearest estimate to the matrix $A \circ \tau^t - S^{t-1}$ with respect to $L^t$ under the current estimate of the parameters $A \circ \tau^t$ and $S^{t-1}$.
The singular value decomposition (SVD) gives a closed-form solution to this problem:
\begin{equation} \label{eq:6}
L^t = \sum_{i=1}^k \sigma_i U_i V_i^T
\end{equation}
where the coefficients $\sigma_i$ and the vectors $U_i$ and $V_i$ ($i = 1,\dots,n$) are the singular values, and the left and right singular vectors of the matrix $A \circ \tau^t - S^{t-1}$, respectively. 
\end{subp2}
\newtheorem*{subp3}{Solving the third optimization problem -- equation (\ref{eq:5})}
\begin{subp3}  \normalfont
Finally the matrix $S^t$ is updated using the parameter $\lambda$ acting as a tuning parameter in the matrix $A \circ \tau^t - L^t$; i.e., the elements of the matrix $A \circ \tau^t - L^t \leq \lambda$ are considered to be zero.
\end{subp3}

\begin{algorithm}   \footnotesize
  \caption{Approximated RPCA with moving camera}
  \label{lst:alg1}
  \begin{algorithmic}[1]
    \STATE \textbf{Input:} $A$, $k$, $\lambda$, \textit{tol}, \textit{maxIter}
    \STATE \textbf{Output:} $S$, $L$, $\tau$
    \STATE \textit{Standard initialization:} $\tau^0 = 0$, $L^0 = A$, $S^0 = 0$
    \WHILE{${\|A \circ \tau^t - L^t - S^t\|_F^2}/{\|A\|_F^2} > \textit{tol}$
    $\text{or} \quad \textit{t} < \textit{maxIter}$}
    {\begin{enumerate}[label=\arabic*)]
    \item Form the matrix $A \circ \tau$ calculating the parameters $\tau_i^t$ that infer the mapping that transforms the column vector $A_i$ to the $i$-th column vector of the matrix $L^{t-1} + S^{t-1}$.
    \item Calculate $L^t = \sum_{i=1}^k \sigma_i U_i V_i^T \quad \text{where} \quad \operatorname*{svd}(A \circ \tau^t - S^{t-1}) = U \Sigma V^T$.
    \item Calculate $S^t = \mathcal{P}_\lambda (A \circ \tau^t - L^t) \quad \text{where} \quad \mathcal{P}_\lambda (x) = \operatorname*{sign}(x)\max(|x| - \lambda, 0)$.
           \end{enumerate}}
    \ENDWHILE
  \end{algorithmic}
\end{algorithm}
\newtheorem*{conv}{Convergence of the iterative process}
\begin{conv} \normalfont
The sequence of values of the objective function $\|A \circ \tau^t - L^t - S^t\|_F^2 + \lambda \|S^t\|_1$, $t = 1,2, \dots, p$ produced by the iterative process is monotonically decreasing for a fixed $\lambda$ converging to a local minimum. The proof is similar to the convergence arguments used by theorem $1$ in~\cite{ZhouTao11GoDec}. The main difference is the addition of a third optimization problem (which involves the parameters of the motion model) that also has a closed-form solution and the values of the sequence keep decreasing in each step.
\end{conv}

%% file: IEEE_03_Methodology_2.tex
The formulation of the background modeling/foreground detection problem using the optimization function (\ref{eq:2}), favors solutions where the matrix $S$ is sufficiently sparse. But, this formulation does not take into account the structure of sparsity in $S$, and it does not yield good results when the sparse pattern involves for example clutters of non-zero entries representing foreground objects. Strictly speaking, in real-world video sequences the foreground pixels do not appear as in a sparse matrix at random and scattered; but rather, they appear in regions corresponding to foreground objects in the scene, in groups of pixels that have a structure.
%

In this section, we propose a mathematical formulation of the problem that favors solutions where the non-zero elements of the matrix $S$ are structured; in other words the non-zero elements appear in non-overlapping blocks (with no pre-specified sizes), where each block can represent the natural shape of a foreground object. In our algorithm the background sequence is again modeled by a low-rank subspace that can gradually change over time; while the moving foreground objects constitute the correlated and contiguous sparse outliers. It is noteworthy that methods that involve a structured sparsity solution such as~\cite{liu2015TIP}, usually impose a block structure by pre-defining a block size for a group of pixels and then perform the norm minimization using a motion-saliency check, that would yield a block structured foreground detection. As a result, the output is always dependent on the tweaked pre-defined block size and as the objects in the video sequence change sizes the foreground support detection becomes more and more unreliable. Here, unlike other block/group/structured sparsity methods our formulation involves solving a $\ell_{2,1}$-norm minimization for the matrix $mat(S_j)$ that corresponds to the sparse part for video frame $I_j$, and the structured sparsity is guaranteed by the norm itself.

In other algorithms exploiting the block-sparsity~\cite{TangNehorai11BlockSparse},~\cite{GuyonBouwmansZahzah12} the matrix $S$ contains mostly zero columns, with several non-zero ones corresponding to foreground elements. In image processing applications this assumption cannot be made (since we assume that the columns of the matrix $S$ correspond to foreground objects in the frames of a video sequence). Assuming that most columns are zero contradicts the definition of sparse matrix. When a whole column in the sparse matrix is zero it means the information in that column is assigned to the low-rank subspace. Moreover, if the video sequence contains foreground objects in all the frames then this assumption does not help.
Instead, it would make sense if the block-sparsity was imposed on the pixels of each video frame rather than a whole column (whole frame) in the matrix $S$. Hence, we solve the minimization problem for block-sparse matrices $mat(S_j)$. In order to rule out ambiguity, in our model the columns of the low-rank matrix $L$ corresponding to the sparse columns are ensured to be zeros. Therefore, this algorithm achieves more robustness than RPCA-PCP~\cite{Candes11} and RPCA-LBD~\cite{TangNehorai11BlockSparse},~\cite{GuyonBouwmansZahzah12} with dynamic backgrounds, varying foreground object sizes, and illumination changes.

Given a data matrix $A$ whose columns are the frames of a video sequence captured by a moving camera and a soft-thresholding parameter $\lambda$, we minimize the following optimization problem that recovers the background and the structured block foreground of the sequence with the matrices $L$ and $S$, respectively.
\begin{equation} \label{eq:8}
\operatorname*{arg\,min}_{\substack{
            S,\tau\\
            rank(L) \leq k}}{\|A \circ \tau - L - S\|_F^2 + \lambda \sum_{j=1}^n \|mat(S_j)\|_{2,1}}
\end{equation}

Other models that involve the $\ell_{2,1}$-norm minimization of the matrix $S$ have been explored in the literature~\cite{TangNehorai11BlockSparse}. Here the $\ell_{2,1}$-norm of the matrix created by the columns of the matrix $S$ is minimized, so that one can use the additional information derived from the spatial positions of the non-zero elements and introduce zero blocks to strategically enforce sparsity.

In the proposed formulation, the first summand guarantees the approximation of the decomposition (minimizing the residual) and the second favors solutions where the zero elements of $S$ appear in blocks corresponding to some columns of $mat(S_j)$. The problem (\ref{eq:8}) is solved by alternatively solving three optimization problems below for $t = 1,2,\dots,p$ until convergence. Algorithm~\ref{lst:alg2} describes the iterative process for the following minimization process.
\begin{equation} \label{eq:9}
\tau^t = \operatorname*{arg\,min}_\tau \|A \circ \tau - L^{t-1} - S^{t-1}\|_F^2
\end{equation}
\begin{equation} \label{eq:10}
L^t = \operatorname*{arg\,min}_{rank(L) \leq k} \|A \circ \tau^t - L - S^{t-1}\|_F^2
\end{equation}
\begin{equation} \label{eq:11}
S^t = \operatorname*{arg\,min}_S \|A \circ \tau^t - L^t - S\|_F^2 + \lambda \sum_{j=1}^n \|mat(S_j)\|_{2,1}
\end{equation}

The first and second optimization problems are similar to the optimization problems solved in the previous section. The solution of the third problem is obtained applying the following lemma to the matrix $H = A \circ \tau^t - L^t$:
%
\newtheorem{lemma}{Lemma}
\begin{lemma}
The $i$-th column of the matrix $mat(E_j)$ that solves the minimization problem
\[E = \operatorname*{arg\,min}_S \|H - S\|_F^2 + \lambda \sum_{j=1}^n \|mat(S_j)\|_{2,1}\]
is given by:
\[\left(mat \left(E_j\right)\right)_i = \left(mat \left(H_j\right)\right)_i \max{\left(0, 1 -  \frac{\lambda}{\|{\left(mat \left(H_j\right)\right)_i}\|_2}\right)}\]
\[j=1,\dots,n \quad \text{and} \quad i=1,\dots,q\]
\end{lemma}

\begin{proof}
The function to be minimized can be written as:
\[\min_S {\sum_{j=1}^n \|mat(H_j) - mat(S_j)\|_F^2 + \lambda \sum_{j=1}^n \|mat(S_j)\|_{2,1}}\]
\[= \min_S {\sum_{j=1}^n \left(\|mat(H_j) - mat(S_j)\|_F^2 + \lambda\|mat(S_j)\|_{2,1}\right)}\]

Applying the \textit{lemma 1} in~\cite{PengiGaneshWrightXu12} in each summand of this expression, we have the closed-from for the $i$-th column of the matrix $mat(E_j)$ given by the lemma.
\end{proof}
\textit{Lemma 1} states that the solution matrix $S^t$ for the third optimization problem has non-overlapping zero blocks in the positions corresponding to the columns of the matrices $mat(S_j)$, $j=1,\dots,n$ with $\ell_2$-norm less than $\lambda$.

\begin{algorithm}   \footnotesize \itemsep0pt \parskip0pt \parsep0pt
  \caption{Block-sparse extension}
  \label{lst:alg2}
  \begin{algorithmic}[1] \footnotesize \itemsep0pt \parskip0pt \parsep0pt
    \STATE \textbf{Input:} $A$, $k$, $\lambda$, \textit{tol}, \textit{maxIter}
    \STATE \textbf{Output:} $S$, $L$, $\tau$
    \STATE \textit{Standard initialization:} $\tau^0 = 0$, $L^0 = A$, $S^0 = 0$
    \WHILE{${\|A \circ \tau^t - L^t - S^t\|_F^2}/{\|A\|_F^2} > \textit{tol}$
    $\text{or} \quad \textit{t} < \textit{maxIter}$}{
    \begin{enumerate}[label=\arabic*)] \itemsep0pt \parskip0pt \parsep0pt
    \item Form the matrix $A \circ \tau$ calculating the parameters $\tau_i^t$ that infer the mapping that transforms the column vector $A_i$ to the $i$-th column vector of the matrix $L^{t-1} + S^{t-1}$.
    \item{Calculate $L^t = \sum_{i=1}^k \sigma_i U_i V_i^T \quad \text{where} \quad \operatorname*{svd}(A \circ \tau^t - S^{t-1}) = U \Sigma V^T$}.
    \item Let $H = A \circ \tau^t - L^t$
    \end{enumerate}}
    		\FOR{$j = 1,\dots,n$}
    			\FOR{$l = 1,\dots,q$}
    						\IF{$\|\left(mat(H_j)\right)_i\|_2 < \lambda$}
    									\STATE{$\left(mat(S_j)\right)_i := 0$}
    							   \ELSE
    							        \STATE{$\left(mat(S_j)\right)_i := 1 - \frac{\lambda \left(mat(H_j)\right)_i}{\|\left(mat(H_j)\right)_i\|_2} $}
    							   \ENDIF
    						
    				   \ENDFOR
    		\ENDFOR  
    \ENDWHILE
  \end{algorithmic}
\end{algorithm}

%

%% file: IEEE_03_Methodology_3.tex
The optimization problems described in Sections~\ref{subsec:GenMovGoDec} and~\ref{subsec:BlockSparse} are solved by iterative procedures that need to be initialized using starting values of the matrices $L$, $S$ and $\tau$. Algorithms~\ref{lst:alg1} and~\ref{lst:alg2} start the iterative process with a standard (na\"{i}ve) initialization of $L^0 = A$, $S^0 = 0$, and $\tau^0 = 0$. However, a better separation of the static part and moving objects of the frames is obtained when the matrices are strategically initialized. In this section, a novel strategy for the initialization in algorithms~\ref{lst:alg1} and~\ref{lst:alg2} is proposed.

The rank-$k$ matrix that is the nearest to the matrix $A$ is a low-rank matrix that gives a good first approximation for the static part of the sequence but some parts of the moving objects remain in this rank-$k$ matrix (for instance when they move slowly, remain inactive for some period of time, or obscure part of the background during the training period). With current RPCA-based optimizations these parts that are called ``ghost" (figure~\ref{fig:RPCALBDvsBS} (b) and (d)) usually persist during the iterative process and are not removed completely. In order to reduce the influence of the ghosting effect during the iterative process one needs to have a prior knowledge of distribution of outliers (which usually appear in clusters of pixels). Hence in this work it is proposed to construct a matrix $S^0$ whose columns contain only the more salient part of the matrices $mat(A_j - L_j^0)$, $j=1,\dots,n$ where $L^0$ is the rank-$k$ matrix approximation of the matrix $A$. In other words, as each matrix $mat(A_j - L_j^0)$ contains a sketch of the moving objects in the $j$-th frame, the idea is forming the initial approximations $S^0$ using the columns of the $mat(A_j - L_j^0)$ that contribute to the the non-uniformity of the structure of the matrix. Here the \textit{leverage} scores are used to measure the importance of the columns of the matrix $mat(A_j - L_j^0)$ (the \textit{leverage} scores can be regarded as a pseudo-motion saliency map). Let the $i$-th column of the matrix to be a linear combination of the orthonormal basis given by the left singular vectors of the matrix $\left(mat(A_j - L_j^0)\right)_i = \sum_{k=1}^{rank} \sigma_k U_k V_k^i$, $i=1,\dots,h$ where $U_k$ is the $k$-th left singular vector, $V_k^i$ is the $i$-th coordinate of the $k$-th right singular vector, and $rank$ is the rank of matrix $mat(A_j - L_j^0)$. As the matrices $mat(A_j - L_j^0)$ are approximations of the frames containing the moving objects, they can be considered as approximations to low-rank matrices. It implies that one can assume:
\begin{equation}
\left(mat(A_j - L_j^0)\right)_i \approx \sum_{k=1}^{p} \sigma_k U_k V_k^i \quad , \quad p \ll rank
\end{equation}
Note that any two columns $i_1$ and $i_2$ differ only in the terms $\sum_{k=1}^p V_k^{i_1}$ and $\sum_{k=1}^p V_k^{i_2}$. Then these terms can be used to measure the importance or contribution of each column to the matrix. The normalized statistical \textit{leverage} scores~\cite{Mahoney2009CUR} of the $i$-th column of the matrix $mat(A_j - L_j^0)$ is defined as:
\begin{equation}
\pi_i = \frac{1}{p} \sum_{k=1}^p V_k^i \quad , \quad i=1,\dots,h
\end{equation}
where $h$ is the number of columns of each frame of the sequence. The sub-index $j$ is removed to help understanding of this expression.
\textit{Leverages} have been used historically for outlier detection in statistical regression but recently they have been used to give a column (or row) order of the amount of motion saliency in a specific part of the image. 
The vector $\pi_i$ is a probability vector, i.e. $\sum_{i=1}^h \pi_i = 1$. Therefore, the columns of each matrix $mat(A_j - L_j^0)$ with leverages greater than $1/h$ are the more important columns. So the columns of the initial approximation $S^0$ contain only the more important columns of the matrices $mat(A_j - L_j^0)$, $j=1,\dots,n$. Consequently, the less salient (more static) parts of the image are not included in the initialization of the sparse part, making the iterative process less prone to converge to the wrong local optimum, yielding more stable results, and increasing the segmentation accuracy.
\noindent
\[\begin{aligned}
 \left(mat(S_j^0)\right)_i =
  \begin{cases}
   \left(mat(A_j - L_j^0)\right)_i & lev((mat(A_j - L_j^0))_i) \geq \frac{1}{h} \\
   0       & \textit{otherwise}
  \end{cases}
  \end{aligned}\]
\noindent


\newtheorem*{InitVals}{Initial values for the parameter vector $\tau$}
\begin{InitVals} \normalfont
To initialize the motion parameter vector, it is a feasible practice to apply the strategy proposed in~\cite{ZhouEtAl13DECOLOR} where the parameters align each frame $A_j$, $j = 1, \dots, n$ to the middle frame of the sequence ($A_{\rfrac{n}{2}}$). Then these initial values of $\tau$ are optimized in the iterative process along with the other parameters.
\end{InitVals}


%% file: IEEE_03_Methodology_4.tex
In this section, a particular case is considered, where the background does not change throughout the sequence. It means, the background can be described by a rank-$1$ matrix. This is particularly useful for applications such as video coding, or surveillance background subtraction in indoor environments, where the background does not change for prolonged periods or a duration of time. Henceforth, the optimization problem to be solved using the approximated RPCA is:
\begin{equation}
\operatorname*{arg\,min}_{\substack{
            S,L,\tau\\
            rank(L)=1\\
            card(S) \leq \kappa}} \|A \circ \tau - L - S\|_F
\end{equation}

Note that the soft-thresholding parameter $\lambda$ is left out, and instead the cardinality (number of non-zero elements) $card(S)$ is being fixed. The cardinality $\kappa$ acts as a hard-thresholding parameter that controls the quality of the reconstruction of $A$ using the matrices $L$ and $S$.
The rank-$1$ restriction for $L$ imposed in the optimization problem yields to solutions where the columns of the matrix $L$ can be written as $L_j \leftarrow \alpha L_1$, $j=1, \dots, n$, where $L_1$ is the first column of $L$ and $\alpha$ is a scalar.

Based on this fact, we assume a particular rank-$1$ matrix $L$ where all the column vectors are equal; i.e. $L = l\mathds{1}^T$ where $l$ is a vector of size $m$ and $\mathds{1}^T = (1, \dots, 1)$. Note that any matrix in the form $l\mathds{1}^T$ is a rank-$1$ matrix but not all rank-$1$ matrices can be written by repeating the same vector in all the columns.

The main advantage of this special rank-$1$ matrix is that we prove the vector $l$ can be calculated without computing a SVD; therefore, this algorithm converges much faster as a result, since the most expensive computation in the described algorithms is in SVD calculation step. The optimization model is as below:
\begin{equation} \label{eq:SVDFreeMin}
\operatorname*{arg\,min}_{\substack{
            S,l,\tau\\
            card(S) \leq \kappa}} \|A \circ \tau - l\mathds{1}^T - S\|_F
\end{equation}
\newtheorem*{SolOpt}{Solving the optimization problem}
\begin{SolOpt} \normalfont
Applying the alternating strategy three optimization problems must be solved for $t=1,2,\dots,p$ until convergence.
\begin{equation}
\tau^t = \operatorname*{arg\,min}_\tau \|A \circ \tau - l^{t-1}\mathds{1}^T - S^{t-1}\|_F^2
\end{equation}
\begin{equation} \label{eq:13}
l^t = \operatorname*{arg\,min}_{l} \|A \circ \tau^t - l\mathds{1}^T - S^{t-1}\|_F^2
\end{equation}
\begin{equation}
S^t = \operatorname*{arg\,min}_{card(S) \leq \kappa} \|A \circ \tau^t - l^t\mathds{1}^T - S\|_F^2
\end{equation}

Algorithm~\ref{lst:alg3} describes the iterative process for the SVD-free algorithm. The first optimization problem is solved as described in section~\ref{subp1:FirstSubProb}. The matrix $S^t$ that solves the third optimization problem is the matrix with zero elements in the positions corresponding to the first $\kappa$ smallest elements of the matrix $|A \circ \tau^t - l^t\mathds{1}^T|$.
\end{SolOpt}

\newtheorem*{subp2sec3.4}{Solving the second optimization problem -- equation (\ref{eq:13})}
\begin{subp2sec3.4} \normalfont
Let $E$ be the matrix $A \circ \tau^t - S^{t-1}$; the following lemma gives a closed-form solution for the solution vector $l$ of the second optimization problem.
\begin{lemma}
The solution $l$ of the optimization problem $\operatorname*{arg\,min}_l \|E - l\mathds{1}^T\|_F^2$ is given by:
\[
l_i = \frac{1}{n} \sum_{j=1}^n E_{ij} \quad , \quad i=1,\dots,m
\]
\end{lemma}
\begin{proof}
Expanding the objective function we have:
\noindent
\begin{multline*}
\begin{aligned}
\|E - l\mathds{1}^T\|_F^2 = \sum_{k=1}^n \|E_k - l\mathds{1}^T\|_F^2 = \\
\left(E_{11} - l_1\right)^2 + \dots + \left(E_{m1} - l_m\right)^2 + \dots + \left(E_{1n} - l_2\right)^2 \\
+ \dots + \left(E_{mn} - l_m\right)^2
\end{aligned}
\end{multline*}
\noindent
Grouping terms together,
\noindent
\begin{multline*}
\begin{aligned}
= \left(E_{11} - l_1\right)^2 + \dots + \left(E_{1n} - l_1\right)^2 + \dots + \left(E_{m1} - l_m\right)^2\\
+ \dots + \left(E_{mn} - l_m\right)^2 = \sum_{i=1}^m \sum_{j=1}^n \left(E_{ij} - l_i\right)^2
\end{aligned}
\end{multline*}
\noindent
Setting the derivatives of each $i$-th term $\sum_{j=1}^m \left(E_{ij} - l_i\right)^2$ to zero with respect to $l_i$ yields:
\noindent
\begin{multline*}
\begin{aligned}
\sum_{j=1}^n \left( E_{ij} - l_i \right)^2 = \\
-2 \left( E_{i1} - l_i \right) - 2 \left( E_{i2} - l_i \right) - \dots - 2 \left( E_{in} - l_i \right) = 0\\
= -2 \left( E_{i1} + \dots + E_{in} \right) + 2 n l_i = 0
\end{aligned}
\end{multline*}
\noindent
We have from there that:
\[l_i = \frac{1}{n} \sum_{j=1}^n E_{ij}\]
\end{proof}
In general the problems $\operatorname*{arg\,min}_l \|E - l\mathds{1}^T\|_F^2$ and $\operatorname*{arg\,min}_{rank(L)=1} \|E - L\|_F^2$ have different solutions. The optimal solution of the first problem is the vector $l = \frac{1}{n} \sum_{j=1}^n E_j$ and the solution of the second one is the matrix $L = \sigma_1 V_1 U_1^T$ where $V_1$ and $U_1$ are the largest right and left singular vectors and $\sigma_1$ is the largest singular value.
Since $l\mathds{1}^T$ is a particular rank-$1$ matrix, we know the value of the objective function of the second problem in the optimal solution is less than or equal to the minimum value of the first problem. The following lemma characterizes this difference in terms of the distance of the one-dimensional linear subspace spanned by the columns of the matrices $l\mathds{1}^T$ and $L$.

\begin{lemma}
Suppose that $E$ is a full rank matrix. Let $l$ and $L$ be the optimal solutions of the problems $\operatorname*{arg\,min}_l \|E - l\mathds{1}^T\|_F^2$ and $\operatorname*{arg\,min}_{rank(L)=1} \|E - L\|_F^2$ respectively; then:
\[
dist \left( span \left\lbrace l \right\rbrace , span \left\lbrace L \right\rbrace \right) = O \left( \frac{\sigma_2}{\sigma_1} \right)
\]
where $span \left\lbrace l \right\rbrace$ and $span \left\lbrace L \right\rbrace$ denote the one-dimensional subspaces spanned by $l$ and $L$ respectively, and $\sigma_1$ and $\sigma_2$ are the two largest singular values of $E$.
\end{lemma}
\begin{proof}
It is based on the proof of the convergence of the power method for computing the eigenvectors of a matrix in~\cite{GolubVanLoan96}. Consider the spectral decomposition of the matrix $E E^T$ which is $E E^T = W \Psi W^T$. Suppose that $q$ is a vector of size $m$ that can be written as a linear combination of the eigenvectors of $E E^T$:
\[q = \alpha_1 W_1 + \dots + \alpha_m W_m\]
We have:
\[E E^T q = E E^T \alpha_1 W_1 + \dots + E E^T \alpha_m W_m\]
\[E E^T q = \alpha_1 \psi_1 W_1 + \dots + \alpha_m \psi_m W_m\]
\begin{equation} \label{eq:I}
E E^T q = \alpha_1 \psi_1 \left( W_1 + \sum_{j=2}^m \frac{\alpha_j}{\alpha_1} \frac{\psi_j}{\psi_1} W_j \right)
\end{equation}
This expression is satisfied in particular for a vector of size $m$, $q$ such that $E^T q = \left(1,\dots,1\right)^T \equiv \mathds{1}$ and $q = \left( q_1,\dots,q_n,0,\dots,0 \right)$. The existence of this vector $q$ is supported by $E$ being a full rank matrix. Since $l = \frac{1}{n} E\mathds{1}$ and substituting in equation (\ref{eq:I}), we have:
\[
l = \frac{1}{n} \alpha_1 \psi_1 \left( W_1 + \sum_{j=2}^m \frac{\alpha_j}{\alpha_1} \frac{\psi_j}{\psi_1} W_j \right)
\]
and we can conclude:
\[
dist \left( span \left\lbrace l \right\rbrace , span \left\lbrace W_1 \right\rbrace \right) = O \left( \frac{\psi_2}{\psi_1} \right)
\]
Finally, the thesis of the lemma is obtained using the known relation between the eigenvectors/eigenvalues of $E E^T$ and the singular vectors and singular values of $E$.
\end{proof}
This lemma asserts that the distance between both subspaces and the quotient $\sigma_2 / \sigma_1$ are functions with similar growth rates. In other words, the distance between the minimum value of $\|E - l\mathds{1}^T\|_F^2$ as a function of $l$ and $\|E - L\|_F^2$ as a function of $L$, depends on the gap between the largest and second largest singular values of the matrix $E$. 
It is clear that the main advantage of the proposed optimization function (\ref{eq:SVDFreeMin}) is that the calculation of the expensive SVD is avoided.
%
%
%

\end{subp2sec3.4}
\begin{algorithm}   \footnotesize \itemsep0pt \parskip0pt \parsep0pt
  \caption{SVD-free solution}
  \label{lst:alg3}
  \begin{algorithmic}[1] \itemsep0pt \parskip0pt \parsep0pt
    \STATE \textbf{Input:} $A$, $k$, $\kappa$, \textit{tol}, \textit{maxIter}
    \STATE \textbf{Output:} $S$, $l$, $\tau$
    \STATE \textit{Standard initialization:} $\tau^0 = 0$, $l = \frac{1}{n} \sum_{j=1}^n A_j$, $S^0 = 0$
    \WHILE{${\|A \circ \tau^t - l\mathds{1}^T - S^t\|_F^2}/{\|A\|_F^2} > \textit{tol} $
    $ \text{or} \quad \textit{t} < \textit{maxIter}$}
    \begin{enumerate}[label=\arabic*)] \itemsep0pt \parskip0pt \parsep0pt
    \item Form the matrix $A \circ \tau$ calculating the parameters $\tau_i^t$ that infer the mapping that transforms the column vector $A_i$ to the $i$-th column vector of the matrix $l^{t-1}\mathds{1}^T + S^{t-1}$.
    \item{Calculate $E = A \circ \tau^t - S^{t-1} \quad , \quad l^t = \frac{1}{n} \sum_{j=1}^n E_j$}
    \item{Form $S^t = \mathcal{P}_{\Omega_\kappa} \left(A \circ \tau^t - L^t \right)$ where $\Omega_\kappa$ is the non-zero subset of the first $\kappa$ largest entities of $|A \circ \tau - l^t\mathds{1}^T|$} 
     \end{enumerate}
    \ENDWHILE
    \STATEx \textit{*The initialization strategy described in section~\ref{subsec:GhostRemove} can be used in this algorithm.}
  \end{algorithmic}
\end{algorithm}

%% file: IEEE_04_ExperimentsResults.tex
The proposed algorithms in section~\ref{sec:Proposal} were implemented and tested in MATLAB \textit{8.3.0.532 (R2014a)} on a 64-bit PC with Intel Core i7-4770 CPU @3.40GHz (single core) and 32GB of RAM. A C++ implementation has also been developed. 
For the evaluations, 7 challenging background subtraction datasets (\cite{CDnet2014}, \cite{BMC2012}, \cite{CM2005}, \cite{SAI2011}, \cite{i2R2004}, \cite{MuHAVi2010}, \cite{Moseg2010Malik}) are tested with our models. They vary in resolution, quality, frame number, and general scene scenarios which guarantee an unbiased and rigorous evaluation. Please refer to table~\ref{tab:datasets} for the details and the challenges present in each sequence. A complete description of challenges is available in~\cite{BouwmansBook2014}.

\begin{table*}[!t] \footnotesize 
\setlength\tabcolsep{2.5pt}
\centering

\caption{Information of the sequences used in experiments.}

\begin{tabular}{@{}c@{}@{}c@{}lp{11cm}}
\hline
Dataset & Sequence Name & size $\times$ {\#}frames & Description of challenges \\ \hline
\multicolumn{1}{l}{\multirow{1}{50pt}{{\tiny Change Detection Workshop 2014}~\cite{CDnet2014}} }                                                                &
\multicolumn{1}{l}{Highway}       & $[$240 $\times$ 320$]$ $\times$ 1700 & Dynamic background, crowded scene, camouflage, shadows, foreground aperture  \\ 
\multicolumn{1}{c}{}                                                                                     &
\multicolumn{1}{l}{Office}        & $[$240 $\times$ 360$]$ $\times$ 2050 & Sleeping foreground, foreground aperture  \\ 
\multicolumn{1}{c}{}                                                                                     &
\multicolumn{1}{l}{Pedestrian}    & $[$240 $\times$ 360$]$ $\times$ 1099 & Shadows, foreground aperture \\ 
\multicolumn{1}{c}{}                                                                                     &
\multicolumn{1}{l}{PETS2006}      & $[$576 $\times$ 720$]$ $\times$ 1200 & Shadows, inserted background object, sleeping foreground object \\ 
[0.04cm]
\multicolumn{1}{l}{\multirow{1}{50pt}{{\tiny Background Models Challenge 2012}~\cite{BMC2012}} }                                                                &
\multicolumn{1}{l}{Video001}      & $[$240 $\times$ 320$]$ $\times$ 63 & Noise, illumination changes, bootstrapping, dynamic background, shadows \\ 
\multicolumn{1}{c}{}                                                                                     &
\multicolumn{1}{l}{Video002}      & $[$288 $\times$ 352$]$ $\times$ 73 & Noise, inserted background object, large foreground objects \\ 
\multicolumn{1}{c}{}                                                                                     &
\multicolumn{1}{l}{Video003}      & $[$240 $\times$ 320$]$ $\times$ 15 & Dynamic background, beginning moving object \\ 
\multicolumn{1}{c}{}                                                                                     &
\multicolumn{1}{l}{Video004}      & $[$240 $\times$ 320$]$ $\times$ 59 & Illumination changes, dynamic background, inserted background object, shadows \\ 
\multicolumn{1}{c}{}                                                                                     &
\multicolumn{1}{l}{Video005}      & $[$240 $\times$ 320$]$ $\times$ 79 & Moved background object, inserted background object, dynamic background, bootstrapping \\ 
\multicolumn{1}{c}{}                        &
\multicolumn{1}{l}{Video006}      & $[$240 $\times$ 320$]$ $\times$ 84 & Noise, illumination changes, camouflage \\ 
\multicolumn{1}{c}{}                        &
\multicolumn{1}{l}{Video007}      & $[$240 $\times$ 320$]$ $\times$ 109 & Illumination changes, noise \\ 
\multicolumn{1}{c}{}                        &
\multicolumn{1}{l}{Video008}      & $[$240 $\times$ 320$]$ $\times$ 54 & Camera jitter, camera automatic adjustments, dynamic background, illumination changes, shadows, noise \\ 
\multicolumn{1}{c}{}                        &
\multicolumn{1}{l}{Video009}      & $[$288 $\times$ 352$]$ $\times$ 49 & Illumination changes, shadows \\ 
[0.04cm]
\multicolumn{1}{l}{\multirow{1}{50pt}{{\tiny Carnegie Mellon 2005}~\cite{CM2005}} }                                                                 &
\multicolumn{1}{l}{CM Data}       & $[$240 $\times$ 360$]$ $\times$ 500 & Camera jitter, dynamic background, camouflage \\ 
[0.25cm]
\multicolumn{1}{l}{\multirow{1}{60pt}{{\tiny Stuttgart 2011}~\cite{SAI2011}} }                                                                &
\multicolumn{1}{l}{Camouflage}    & $[$600 $\times$ 800$]$ $\times$ 600 & Camouflage, shadows, dynamic background \\ 
[0.04cm]
\multicolumn{1}{l}{\multirow{1}{50pt}{{\tiny Perception 2004 (i2R)}~\cite{i2R2004}} }                                                                &
\multicolumn{1}{l}{Hall}          & $[$144 $\times$ 176$]$ $\times$ 3584 & Sleeping foreground object, shadows, inserted background object \\ 
\multicolumn{1}{c}{}                                                                                     & 
\multicolumn{1}{l}{Bootstrap}     & $[$120 $\times$ 160$]$ $\times$ 3055 & Bootstrapping, shadows, camouflage, foreground aperture \\ 
\multicolumn{1}{c}{}                                                                                     &
\multicolumn{1}{l}{Curtain}       & $[$218 $\times$ 180$]$ $\times$ 2964 & Dynamic background, camouflage, illumination changes \\ 
\multicolumn{1}{c}{}                                                                                      &
\multicolumn{1}{l}{Escalator}     & $[$130 $\times$ 160$]$ $\times$ 3417 & Dynamic background, noise, illumination changes \\ 
\multicolumn{1}{c}{}                                                                                     &
\multicolumn{1}{l}{Fountain}      & $[$128 $\times$ 160$]$ $\times$ 523 & Dynamic background, camouflage \\ 
\multicolumn{1}{c}{}                                                                                     &
\multicolumn{1}{l}{Shopping Mall} & $[$256 $\times$ 320$]$ $\times$ 1286 & Sleeping foreground object, shadows, inserted background object \\ 
\multicolumn{1}{c}{}                                                                                     &
\multicolumn{1}{l}{Water Surface} & $[$128 $\times$ 160$]$ $\times$ 633 & Dynamic background, foreground aperture, camouflage \\ 
[0.04cm]
\multicolumn{1}{l}{\multirow{1}{60pt}{{\tiny MuHAVi-MAS}~\cite{MuHAVi2010}} }                                                                 &
\multicolumn{1}{l}{Walk Turn Back} & $[$576 $\times$ 720$]$ $\times$ 466 & Camera jitter, camouflage, moved background object \\ 

\hline
\end{tabular}
\label{tab:datasets}
\end{table*}

%% file: IEEE_04_ExperimentsResults_1.tex
Figure~\ref{fig:RPCALBDvsBS} (b)-(e) shows a comparison between the background and foreground parts calculated with the model introduced in section~\ref{subsec:GenMovGoDec} with and without the ghost-removal algorithm introduced in section~\ref{subsec:GhostRemove}. Notice the absorption of some foreground parts in the background when the initialization algorithm is not used in columns (b) and (d). This in turn corrupts the calculated foreground as well. The ghost-removal method significantly reduces the unwanted effect and noise in the results. The results from our algorithm shown in this figure are without any further refinement or morphological operations and solely the raw output of the algorithms. The parameters used to obtain these results are $rank(L) = 1$, $\lambda = 0.1$, and $maxIterations = 10$. Table~\ref{tab:GRTime} shows computing time for different sequences with and without the ghost removal initialization. Note that the initialization process is computationally cheap and the algorithm keeps the same order in time consumption.

\begin{table}[!t]\footnotesize
\centering

\caption{Time consumption for model in~\ref{subsec:GenMovGoDec} with and without ghost removal algorithm in ~\ref{subsec:GhostRemove}. CPU times are in seconds and for 15 iterations.}

\begin{tabular}{llrr} 
\hline
Sequence & size$\times${\#}frames & Original & GR\\
\hline
Water Surface~\cite{i2R2004} & $[$128 $\times$ 160$]$ $\times$ 48  & 2.88  & 2.98\\
Pedestrian~\cite{CDnet2014}  & $[$158 $\times$ 238$]$ $\times$ 24  & 2.77  & 2.83\\
QMUL Junction                & $[$288 $\times$ 360$]$ $\times$ 300 & 73.65 & 73.74\\
\hline
\end{tabular}
\label{tab:GRTime}
\end{table}

%% file: IEEE_04_ExperimentsResults_2.tex
In this experiment the performance of a model named RPCA-LBD~\cite{TangNehorai11BlockSparse} that involves a $\ell_{2,1}$-norm minimization, and the proposed block-sparse model in this paper is evaluated and compared. Figure~\ref{fig:RPCALBDvsBS} shows the results for both algorithms. Columns (b) and (d) correspond to the low-rank and sparse parts obtained by the RPCA-LBD for $\lambda=0.6710$ (as tuned by authors in their original paper). Notice the ghosting effect present in both $L$ and $S$ along with unwanted noise and parts from the dynamic background absorbed into the foreground. Columns (c) and (e) show the results for the same frames obtained by our block-sparse solution with the ghost-removal initialization intact. Notice the ghosting effect which is eliminated and less of the unwanted noise and dynamic background has leaked into the foreground. Finally the columns (f) and (g) show the foreground masks obtained from the sparse parts of both algorithms. The foreground mask for RPCA-LBD in column (f) has been refined with a thresholding strategy explained in section~\ref{subsec:FGsegEval}, and the foreground mask of our method in column (g) is unrefined. The parameters used to obtain these results are $rank(L) = 1$, $\lambda = 0.03$, and $maxIterations = 10$.


%% file: IEEE_04_ExperimentsResults_3.tex
Figure~\ref{fig:SVDfreeReconstComp} shows the visual results for video reconstruction for the proposed SVD-free algorithm which can handle moving cameras. In this example only the vector $l$ (corresponding to background), a cardinality of 15\% of the pixels of the matrix $S$ (corresponding to moving objects in foreground), and the parameter vector $\tau$ were used for reconstructing the original video sequence. Notice that with a small amount of information the sequence is well constructed in the output on column (b). The evaluation of compression ratio gained with this method is out of scope of this work and must be studied further. This sequence has been captured with a moving camera, and therefore the frames are accordingly aligned during the iterative process. Table~\ref{tab:CPUTIME} lists the CPU time for computation of a foreground from the sequence that obtains the maximal performance (in terms of accuracy of foreground detection) for all the videos in six datasets. Notice targeting the same performance, the speed-up in computation time using our SVD-free algorithm.

%% file: IEEE_04_ExperimentsResults_4.tex
For quantitative evaluation, the accuracy of foreground detection is measured by comparing the calculated foreground support $F \in {\lbrace0,1\rbrace}^{m,n}$ with the binary ground truth images.
\begin{equation}
F_{ij} =
  \begin{cases}
   0, & \text{if} \quad S_{ij} \quad \text{is background} \\
   1, & \text{if} \quad S_{ij} \quad \text{is foreground}
  \end{cases}
\end{equation}

This performance measure is regarded as a classification problem and evaluate the results using Precision and Recall values which are defined as:
\[ Precision = \frac{tp}{tp + fp} \]

\[ Recall = \frac{tp}{tp + fn} = \frac{\text{\#correctly classified foreground pixels}}{\text{\#foreground pixels in GT}}\]

The Precision and Recall values are calculated with the number of pixels that are assigned True-Positive $tp$, True-Negative $tn$, False-Positive $fp$, and False-Negative $fn$. Precision and Recall are widely used when the class distribution is skewed~\cite{PrecRecall2006}. In addition a single measurement named $F_{1}$ score is provided, which is the harmonic mean of the Precision and Recall values as:

\[F_{1} = 2 \times \frac{Precision \times Recall}{Precision + Recall}\]

The higher the $F_{1}$ score, the better the detection accuracy. Also the False-Positive Rate is defined as $FPR = \frac{fp}{fp + tn}$.

Figure~\ref{fig:SegCompJunc} shows the unrefined segmentation results for a general surveillance sequence. Notice the accuracy and coherence of the segmentation in our results as compared to that of GoDec. The proposed algorithm can handle objects that occupy large portions of the frame as well as small objects (such as pedestrians in this scene) equally well simultaneously. Table~\ref{tab:FMeasureCompTable} lists the accuracy measures for foreground segmentation in six datasets for our method, RPCA solved via inexact Augmented Lagrange Multipliers\footnote{\url{http://perception.csl.illinois.edu/matrix-rank/sample_code.html}}~\cite{lin2010augmented}, and GoDec\footnote{\url{https://sites.google.com/site/godecomposition/code}}~\cite{ZhouTao11GoDec}. The results for our method are obtained by comparing the obtained foreground (without any further refinement) to the ground truth. For RPCA-PCP method the unrefined support of the sparse matrix would produce a lot of false positives. Therefore, to obtain a more fair comparison for the
RPCA-PCP method, a threshold criterion is required to get the final foreground mask, and the same threshold strategy as in~\cite{gao2012block} is adopted. It is assumed that the distribution of the difference values between the $A$ and $L$ at the tentatively identified background locations, can be an estimation of the expected level of noise; i.e. they satisfy the Gaussian distribution $(\mu, \sigma)$. Thereafter, the threshold is set at the mean of those difference values plus three standard deviations of difference values, and is then applied to $S$ to obtain the foreground support. This is known as the $99.7$ or alternatively, the three-$\sigma$ rule of thumb in statistics, which states that even for non-normally distributed variables, at least 98$\%$ of cases should fall within properly-calculated three-$\sigma$ intervals. Given an observation of difference values $\delta$ the probability distribution can be expressed as:
\[Pr(\mu - 3\sigma \leq \delta \leq \mu + 3\sigma) \approx 0.9973\]

The Receiver Operating Characteristic (ROC) curves obtained with precision-recall values in figure~\ref{fig:Precision-Recall} show the performance of our method against GoDec for varying thresholds. ROC curves are a good tool to graphically illustrate the performance of a binary classifier system as its discrimination threshold is varied. The results here guarantee superior performance for all datasets in foreground segmentation accuracy.
\begin{table*}[!t] \footnotesize
\setlength\tabcolsep{2.85pt}

\centering

\caption{$F_{1}$ scores comparison for foreground detection accuracy. Results from our method are obtained with unrefined raw outputs. RPCA via IALM and GoDec results are refined with a thresholding strategy for a fairer comparison. (Best: bold-face, Second best: underlined). 
}

\begin{tabular}{@{}c@{}@{}c@{}llllllllllll}
\hline
& & \multicolumn{3}{c}{Recall (TPR)} & \multicolumn{3}{c}{Fallout (FPR)} & \multicolumn{3}{c}{Precision (PPV)} & \multicolumn{3}{c}{$F_1$ Score} \\ \cline{3-14}
& & {\tiny {RPCA}} & {\tiny {Ours}} & {\tiny {GoDec}} \hspace{2pt} & {\tiny {RPCA}} & {\tiny {Ours}} & {\tiny {GoDec}} \hspace{2pt} & {\tiny {RPCA}} & {\tiny {Ours}} & {\tiny {GoDec}} \hspace{2pt} & {\tiny {RPCA}} & {\tiny {Ours}} & {\tiny {GoDec}} \\ \hline
\multicolumn{1}{l}{\multirow{1}{80pt}{{\tiny Change Detection Workshop 2014}~\cite{CDnet2014}} }                                                                &
\multicolumn{1}{l}{Highway}       & 0.9136 & 0.7690 & 0.1724 \hspace{2pt} & 0.3438 & 0.0088 & 0.0000 \hspace{2pt} & 0.2026 & 0.8929 & 0.9975 \hspace{2pt} & 0.3316 & \textbf{0.8264} & 0.2940 \\ 
\multicolumn{1}{c}{}                                                                                     &
\multicolumn{1}{l}{Office}        & 0.7405 & 0.8631 & 0.0796 \hspace{2pt} & 0.0637 & 0.0359 & 0.0030 \hspace{2pt} & 0.5325 & 0.7020 & 0.7254 \hspace{2pt} & 0.6195 & 0.7743 & 0.1435 \\ 
\multicolumn{1}{c}{}                                                                                     &
\multicolumn{1}{l}{Pedestrian}   & 0.9489 & 0.7123 & 0.7550 \hspace{2pt} & 0.5067 & 0.0015 & 0.0044 \hspace{2pt} & 0.0415 & 0.9186 & 0.7977 \hspace{2pt} & 0.0796 & \underline{0.8024} & 0.7758 \\ 
\multicolumn{1}{c}{}                                                                                     &
\multicolumn{1}{l}{PETS2006}     & 0.9892 & 0.7197 & 0.1579 \hspace{2pt} & 0.8416 & 0.0028 & 0.0000 \hspace{2pt} & 0.0218 & 0.8278 & 0.9996 \hspace{2pt} & 0.0426 & 0.7700 & 0.2728 \\ 
[0.08cm]
\multicolumn{1}{l}{\multirow{1}{80pt}{{\tiny Background Models Challenge 2012}~\cite{BMC2012}} }                                                                &
\multicolumn{1}{l}{Video001}      & 0.8426 & 0.4021 & 0.9182 \hspace{2pt} & 0.2112 & 0.0141 & 0.3647 \hspace{2pt} & 0.0506 & 0.2763 & 0.0325 \hspace{2pt} & 0.0955 & 0.3275 & 0.0629 
\\ 
\multicolumn{1}{c}{}                                                                                     &
\multicolumn{1}{l}{Video002}      & 0.8224 & 0.6526 & 0.8790 \hspace{2pt} & 0.3380 & 0.0246 & 0.1917 \hspace{2pt} & 0.1648 & 0.6825 & 0.2710 \hspace{2pt} & 0.2746 & 0.6672 & 0.4143 
\\ 
\multicolumn{1}{c}{}                                                                                     &
\multicolumn{1}{l}{Video003}      & 0.9710 & 0.7816 & 0.9308 \hspace{2pt} & 0.1174	& 0.0023 & 0.4575 \hspace{2pt} & 0.1517 & 0.8805	& 0.0421 \hspace{2pt} & 0.2625 & \textbf{0.8281} & 0.0806 
\\ 
\multicolumn{1}{c}{}                                                                                     &
\multicolumn{1}{l}{Video004}      & 0.7305 & 0.6870 & 0.9397 \hspace{2pt} & 0.5969 & 0.0028 & 0.3905 \hspace{2pt} & 0.0135 & 0.7367 & 0.0263 \hspace{2pt} & 0.0266 & \underline{0.7110} & 0.0511 
\\ 
\multicolumn{1}{c}{}                                                                                     &
\multicolumn{1}{l}{Video005}      & 0.9347 & 0.3995 & 0.9627 \hspace{2pt} & 0.4031 & 0.0141 & 0.2921 \hspace{2pt} & 0.0156 & 0.1622 & 0.0220 \hspace{2pt} & 0.0307 & 0.2307 & 0.0430
\\ 
\multicolumn{1}{c}{}                        &
\multicolumn{1}{l}{Video006}      & 0.7613 & 0.6118 & 0.9016 \hspace{2pt} & 0.1135	& 0.0206 & 0.2543 \hspace{2pt} & 0.2102 & 0.5411 & 0.1233 \hspace{2pt} & 0.3294 & 	0.5743 & 0.2169 
\\ 
\multicolumn{1}{c}{}                        &
\multicolumn{1}{l}{Video007}      & 0.3010 & 0.5410 & 0.8170 \hspace{2pt} & 0.3027 & 0.0111 & 0.1735 \hspace{2pt} & 0.0642 & 0.7704 & 0.2452 \hspace{2pt} & 0.1058 & 	0.6356 & 0.3772 
\\ 
\multicolumn{1}{c}{}                        &
\multicolumn{1}{l}{Video008}      & 0.8030 & 0.3650 & 0.8946 \hspace{2pt} & 0.3179 & 0.0067 & 0.4234 \hspace{2pt} & 0.0549 & 0.5575 & 0.0464 \hspace{2pt} & 0.1028 & 	0.4412 & 0.0882
 \\ 
\multicolumn{1}{c}{}                        &
\multicolumn{1}{l}{Video009}      & 0.9157 & 0.6052 & 0.9193 \hspace{2pt} & 0.8348 & 0.0004 & 0.3600 \hspace{2pt} & 0.0046 & 0.8485 & 0.0106 \hspace{2pt} & 0.0091 & 	0.7065 & 0.0210
 \\ 
[0.08cm]
\multicolumn{1}{l}{\multirow{1}{80pt}{{\tiny Carnegie Mellon 2005}~\cite{CM2005}} }                                                                 &
\multicolumn{1}{l}{CM Data}       & 0.7861 & 0.7271 &	0.7232 \hspace{2pt} & 0.3076 & 0.0050 & 0.0281 \hspace{2pt} & 0.0485 & 0.7428 & 0.3394 \hspace{2pt} & 0.0913 &	\textbf{0.7349} & \underline{0.4620}
 \\ 
[0.08cm]
\multicolumn{1}{l}{\multirow{1}{80pt}{{\tiny Stuttgart 2011}~\cite{SAI2011}} }                                                                &
\multicolumn{1}{l}{Camouflage}    & 0.8653 & 0.6123 & 0.1310 \hspace{2pt} & 0.0696 & 0.0204 & 0.0026 \hspace{2pt} & 0.2621 & 0.4620 & 0.5856 \hspace{2pt} & \underline{0.4023} & \textbf{0.5266} & 0.2141
 \\ 
[0.08cm]
\multicolumn{1}{l}{\multirow{1}{80pt}{{\tiny Perception 2004 (i2R)}~\cite{i2R2004}} }                                                                &
\multicolumn{1}{l}{Hall}          & 0.9036 & 0.6516 &	0.3656 \hspace{2pt} & 0.2042	& 0.0160 & 0.0009 \hspace{2pt} & 0.2315 & 0.7347 & 0.9667 \hspace{2pt} & 0.3686 &	0.6907 & 0.5305
 \\ 
\multicolumn{1}{c}{}                                                                                     & 
\multicolumn{1}{l}{Bootstrap}     & 0.5501 & 0.6386 &	0.3142 \hspace{2pt} & 0.0949	& 0.0257 & 0.0021 \hspace{2pt} &	0.3562 & 0.7035	& 0.9349 \hspace{2pt} & 0.4324 &	0.6695 & 0.4703
 \\ 
\multicolumn{1}{c}{}                                                                                     &
\multicolumn{1}{l}{Curtain}       & 0.9736 & 0.8353 &	0.8634 \hspace{2pt} & 0.7900	& 0.0139 & 0.0159 \hspace{2pt} &	0.1081 & 0.8555	& 0.8422 \hspace{2pt} & 0.1946 &	0.8452 & \underline{0.8526}
 \\ 
\multicolumn{1}{c}{}                                                                                      &
\multicolumn{1}{l}{Escalator}     & 0.7506 & 0.5238 &	0.4873 \hspace{2pt} & 0.1234	& 0.0108 & 0.0080 \hspace{2pt} &	0.2170 & 0.6883	& 0.7354 \hspace{2pt} & 0.3367 &	0.5949 & 0.5862
 \\ 
\multicolumn{1}{c}{}                                                                                     &
\multicolumn{1}{l}{Fountain}      & 0.9578 & 0.5803 &	0.8706 \hspace{2pt} & 0.4864	& 0.0071 & 0.1647 \hspace{2pt} &	0.0654 & 0.7439	& 0.1581 \hspace{2pt} & 0.1224 &	0.6520 & 0.2675
 \\ 
\multicolumn{1}{c}{}                                                                                     &
\multicolumn{1}{l}{Shopping Mall} & 0.9773 & 0.6973 &	0.2287 \hspace{2pt} & 0.7038	& 0.0167 & 0.0010 \hspace{2pt} &	0.0831 & 0.7318	& 0.9343 \hspace{2pt} & 0.1531 &	0.7142 & 0.3675
 \\ 
\multicolumn{1}{c}{}                                                                                     &
\multicolumn{1}{l}{Water Surface} & 0.9880 & 0.7967 &	0.9468 \hspace{2pt} & 0.6781	& 0.0026 & 0.1249 \hspace{2pt} &	0.1125 & 0.9644	& 0.3973 \hspace{2pt} & 0.2020 &	\textbf{0.8726} & 0.5598
 \\ 
[0.08cm]
\multicolumn{1}{l}{\multirow{1}{80pt}{{\tiny MuHAVi-MAS}~\cite{MuHAVi2010}} }                                                                 &
\multicolumn{1}{l}{Walk Turn Back} & 0.9655 & 0.8617 & 0.9203 \hspace{2pt} & 0.1852 & 0.0030 & 0.0277 \hspace{2pt} & 0.0681 & 0.8007 & 0.3177 \hspace{2pt} & 0.1272 & \textbf{0.8301} & \underline{0.4724} \\ 
\hline
\multicolumn{2}{c}{\multirow{1}{*}{Average} }                                                                &
 0.8431 & 0.6537 & 0.6600 \hspace{2pt} &	0.3754 & 0.0116 & 0.1431 \hspace{2pt} & 0.1340 &	0.7054 & 0.4587 \hspace{2pt} & 0.2061 & \textbf{0.6707} &	\underline{0.3315}
 \\
\hline
\end{tabular}
\label{tab:FMeasureCompTable}
\end{table*}
\begin{table*}[!t] \footnotesize
\centering

\caption{Time performance comparison of the RPCA via IALM, the approximated RPCA GoDec, and our SVD-Free approximated RPCA for processing the whole dataset videos. All algorithms were run with 5 iterations and the parameter $\lambda$ was chosen to obtain maximal $F_1$ measure performance.}

\begin{tabular}{lllllll}
\hline
& {CDW~\cite{CDnet2014}} & {BMC~\cite{BMC2012}} & {CM~\cite{CM2005}} & {SAI~\cite{SAI2011}} & {i2R~\cite{i2R2004}} & {MuHAVi-MAS~\cite{MuHAVi2010}}\\ \cline{2-7}  
Number of Frames & {6049} & {591} & {500} & {600} & {15462} & {466} \\ \hline
RPCA (CPU sec.) & 1931.12 & 116.94 & 65.67 & 454.59 & 646.52 & 380.85 \\
[0.08cm]           
GoDec (CPU sec.) & 1874.82  & 49.67 & 44.26 & 376.91 & 480.33 & 203.17 \\
[0.08cm]
Our Model (CPU sec.)       & 358.91   & 20.82 & 17.38 & 117.69 & 209.84 & 70.27 \\
\hline
\end{tabular}
\label{tab:CPUTIME}
\end{table*}
\begin{figure*}[!t] 
\centering
\includegraphics[trim =1.4cm 1.8cm 1.4cm 1.2cm, clip = true,width=0.98\textwidth]{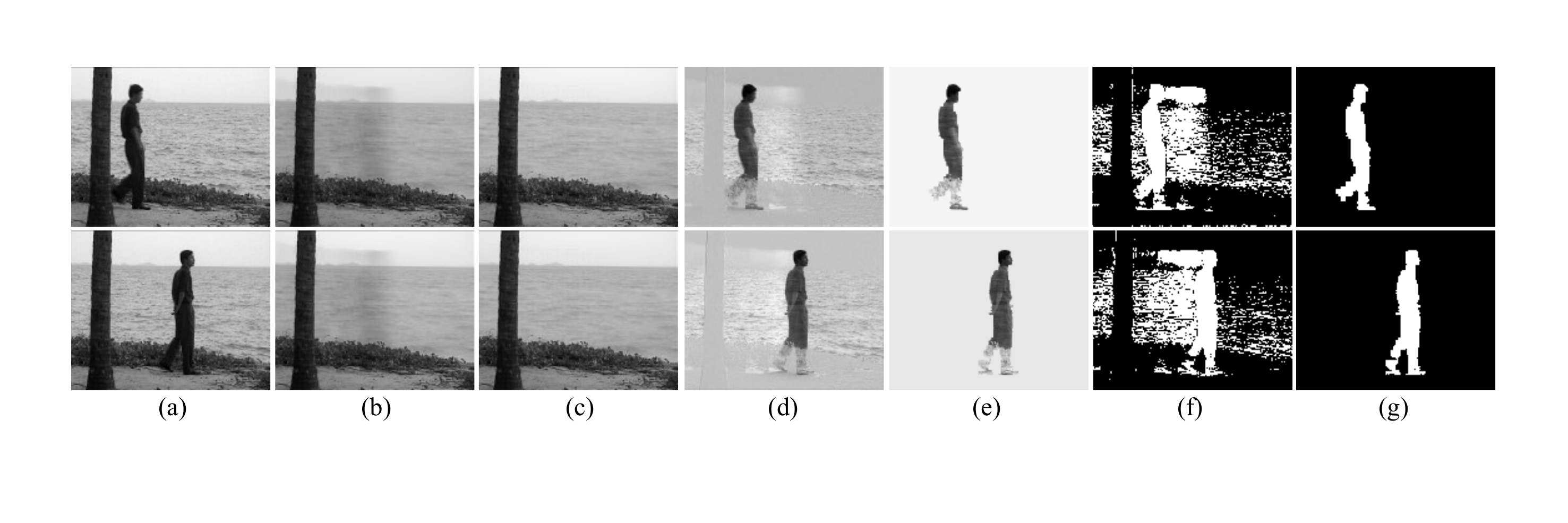}
    \caption{Effect of initialization (b)-(e) and RPCA-LBD algorithm vs. our block-sparse model (d)-(g), comparison for a complex video sequence with dynamic background (water rippling) and camouflage. (a) Original video for frames 24 and 48. (b) Corresponding background extracted by RPCA-LBD. (c) Background with our block-sparse method. (d) Foreground by RPCA-LBD. (e) Foreground with our block-sparse method. (f) Refined foreground mask obtained by RPCA-LBD. (g) Unrefined foreground mask obtained by our block-sparse method.}
\label{fig:RPCALBDvsBS}
\end{figure*}
\begin{figure*}[!t] 
\centering
\includegraphics[trim =2cm 2.4cm 2cm 2.6cm, clip = true,width=0.80\textwidth]{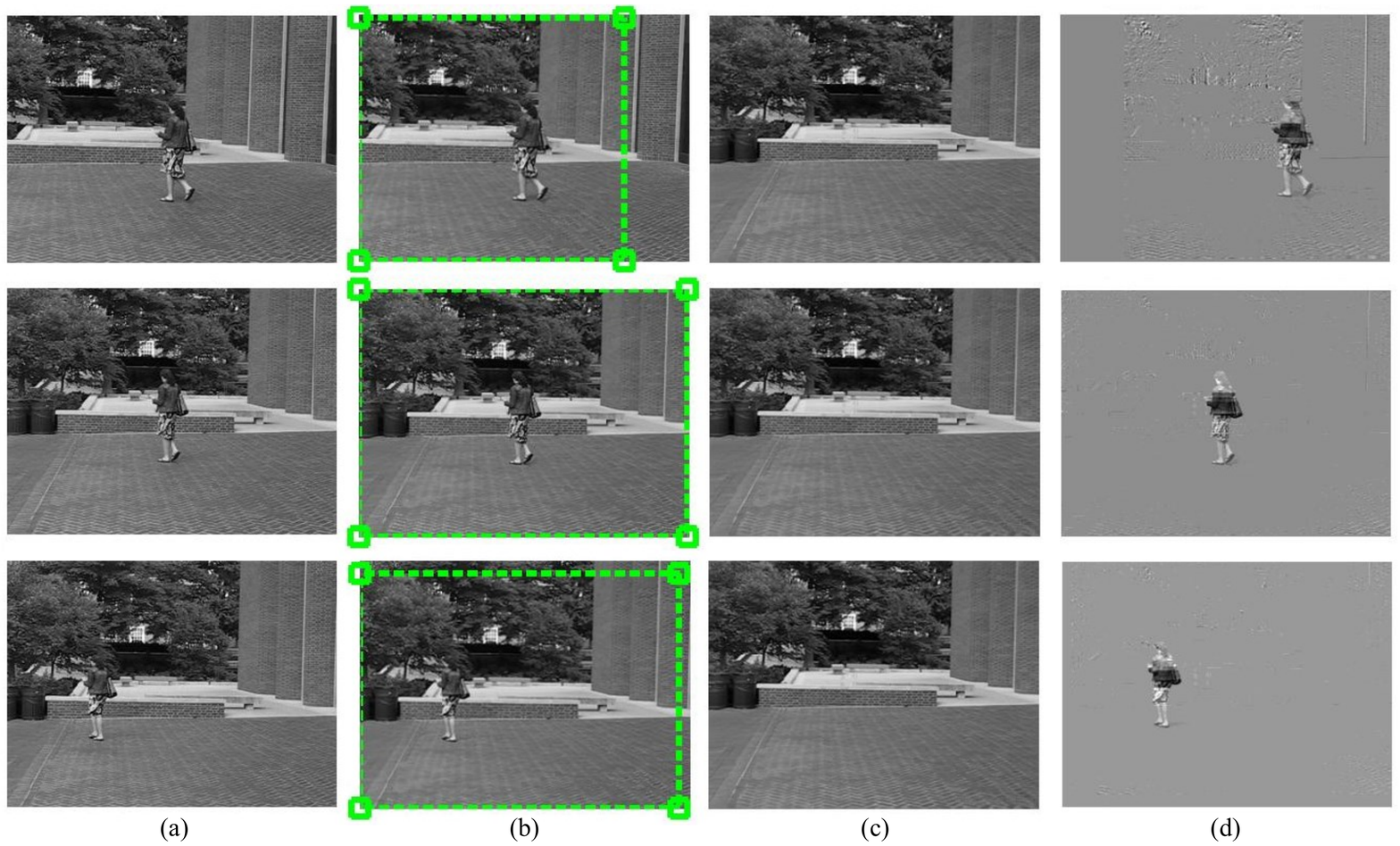}
    \caption{Video reconstruction of SVD-free algorithm, a sequence with moving camera~\cite{Moseg2010Malik}. (a) Original video for frames 1, 20, and 40. (b) Reconstruction results with $L + S$. The marked green region corresponds to the recovered rank-$1$ background across the whole sequence with the alignment procedure described (i.e. transformed images) with motion parameters. (c) Motion-compensated extracted background $L$. (d) Motion-compensated extracted foreground $S$.}
\label{fig:SVDfreeReconstComp}
\end{figure*}
\begin{figure*}[!t] 
\centering
\includegraphics[trim =1.5cm 3.5cm 1.5cm 4cm, clip = true,width=0.80\textwidth]{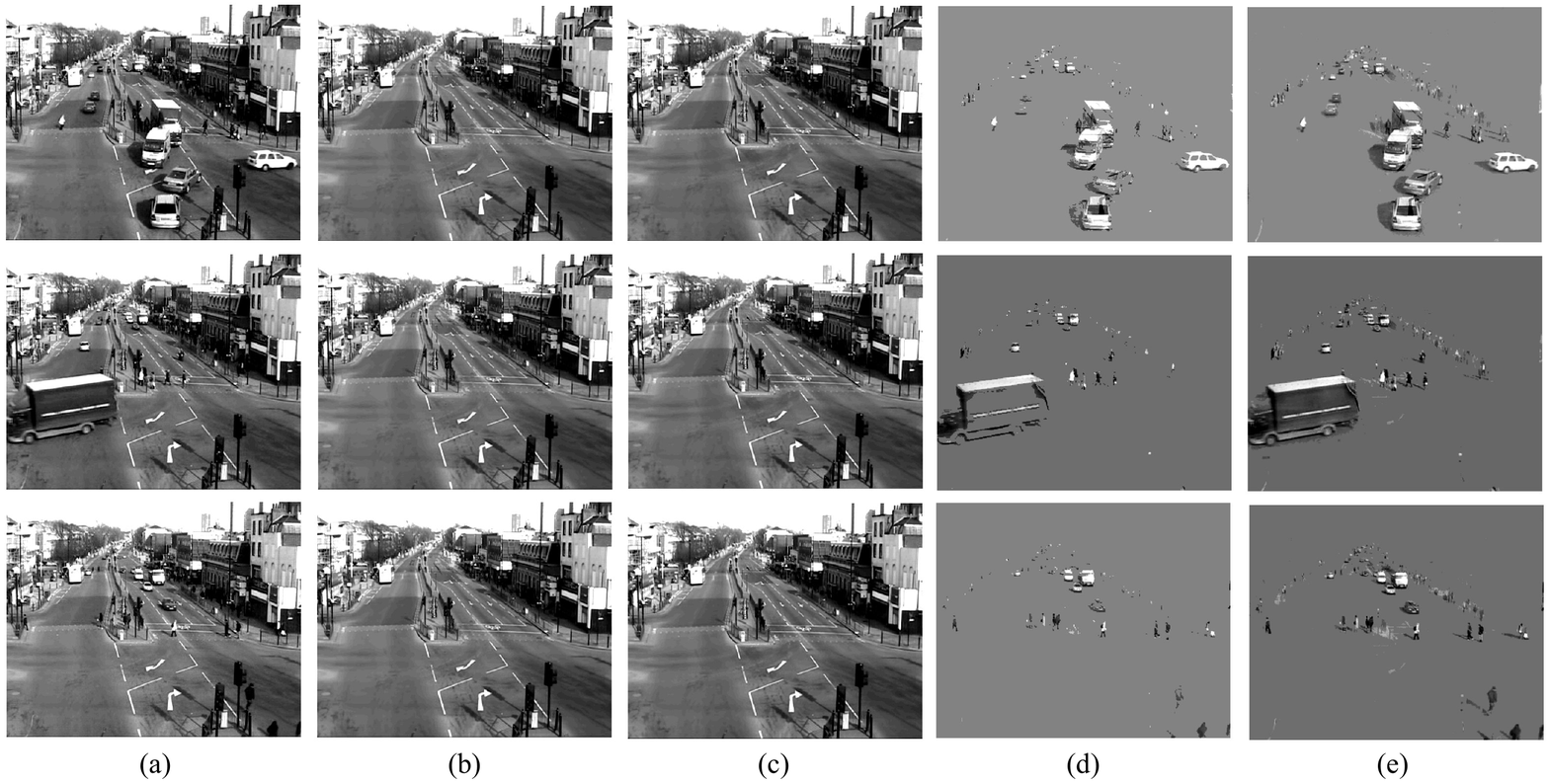}
   \caption{Segmentation results comparison. Sequence from QMUL Junction dataset. (a) Original frames, (b) Low-rank GoDec, (c) Low-rank ours, (d) Sparse GoDec, (e) Sparse ours. Notice the quality of segmentation and details recovered by our model.}
\label{fig:SegCompJunc}
\end{figure*}
\begin{figure*}[!t]
\centering
\begin{subfigure}{.32\textwidth}
  \centering
  \includegraphics[trim =3.1cm 4.8cm 10.6cm 1cm, clip = true,width=1\linewidth]{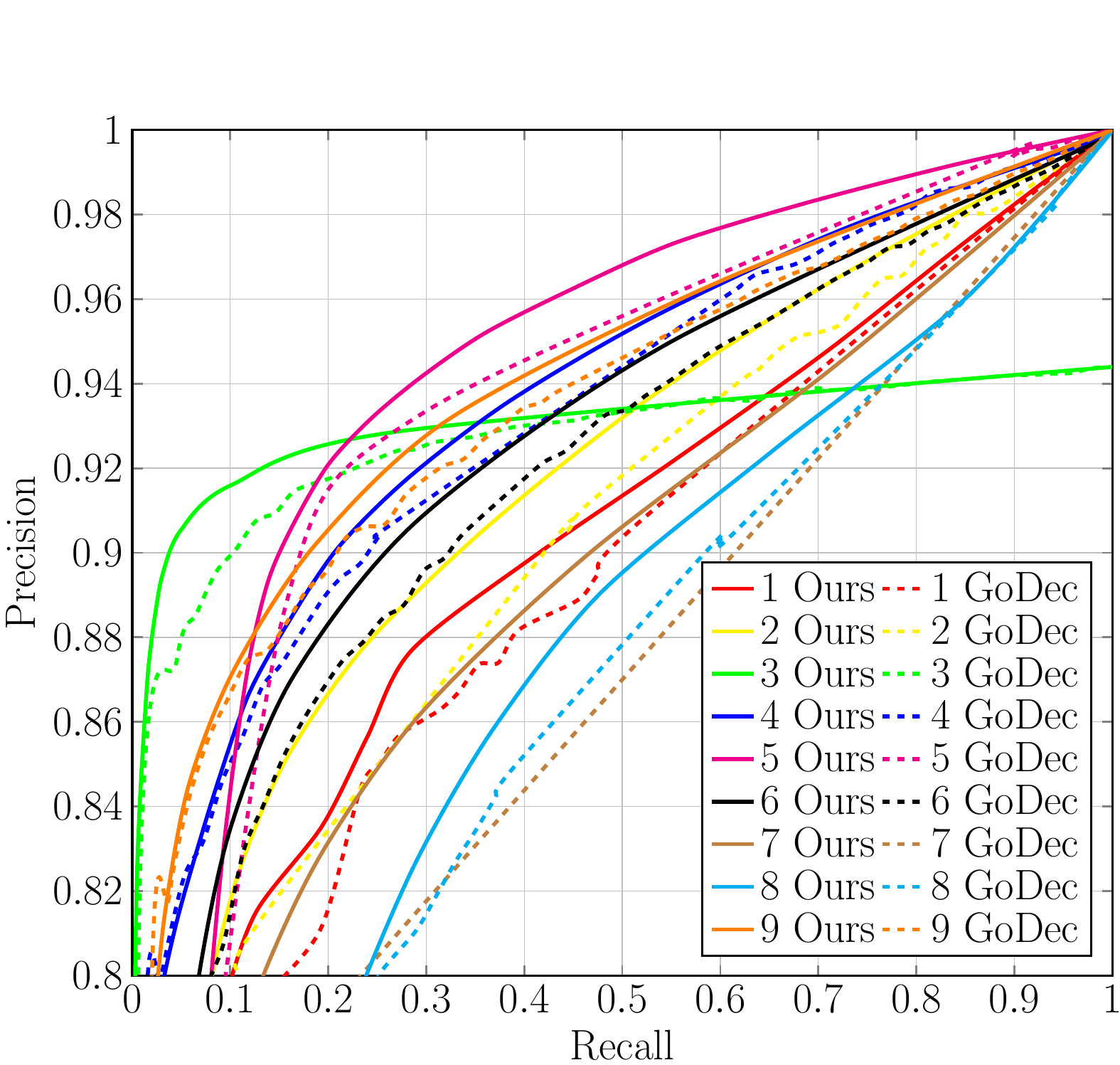}
  \caption{{\footnotesize BMC}}
  \label{fig:sfig1}
\end{subfigure}
\begin{subfigure}{.32\textwidth}
  \centering
  \includegraphics[trim =3.1cm 4.8cm 10.6cm 1cm, clip = true,width=1\linewidth]{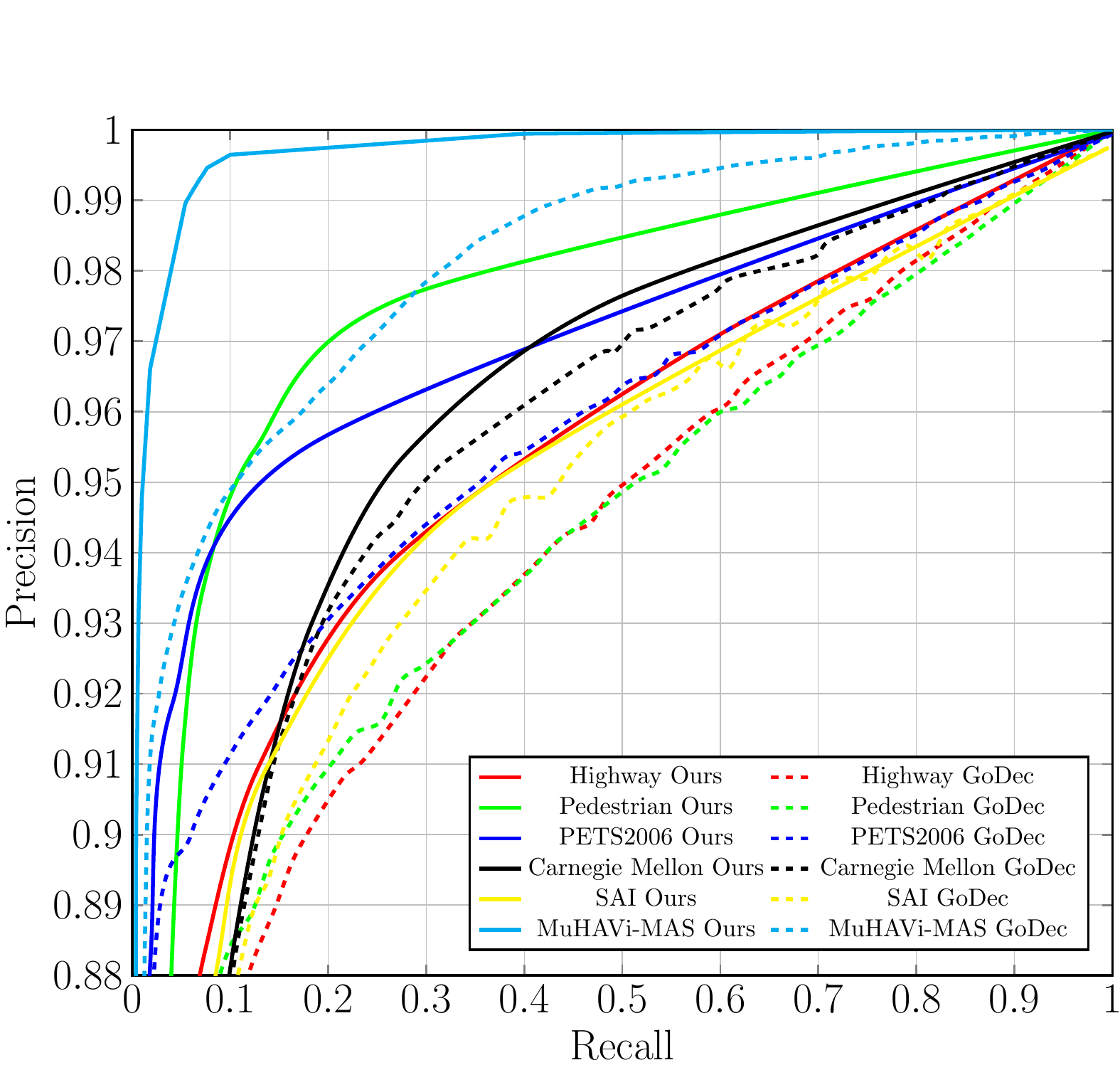}
  \caption{{\footnotesize CDW, CM, SAI, and MuHAVi-MAS}}
  \label{fig:sfig2}
\end{subfigure}
\begin{subfigure}{.32\textwidth}
  \centering
  \includegraphics[trim =3.1cm 4.8cm 10.6cm 1cm, clip = true,width=1\linewidth]{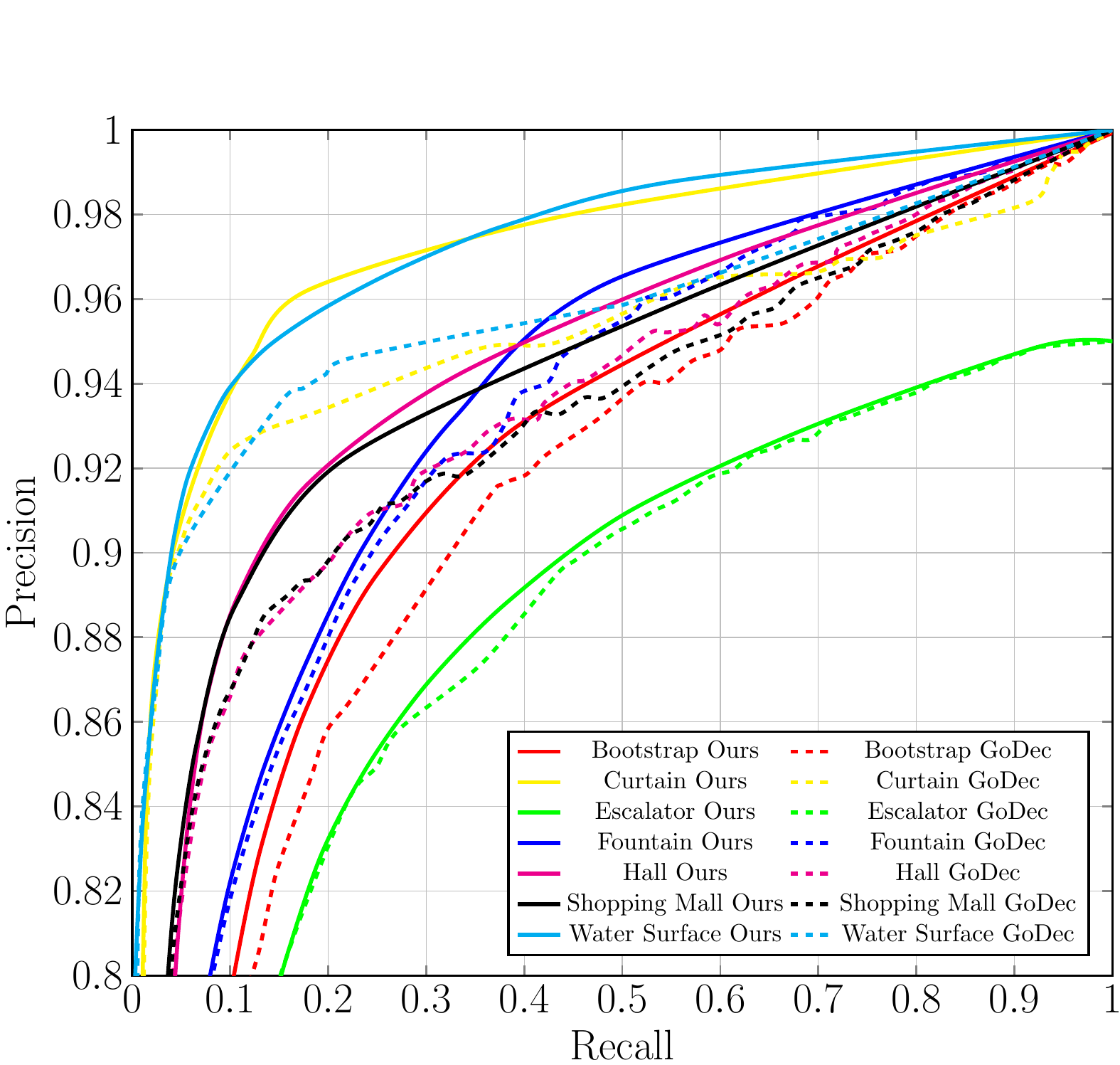}
  \caption{{\footnotesize i2R}}
  \label{fig:sfig3}
\end{subfigure}
\caption{Receiver Operating Characteristic (ROC) curves for the performance of our method vs. GoDec with varying thresholds.}
\label{fig:Precision-Recall}
\end{figure*}

\subsection{Removing Shadows and Specularities and Aligning Face Images}
Our model can be extended to applications such as piece-wise face image alignment. Given $n$ images $I_1, \dots, I_n$ of an object which are misaligned the optimization problem is solved in such a way that the resulting images $I_1 \circ \tau_1, \dots, I_n \circ \tau_n$ are well-aligned. This application has been proposed in a work by Peng \textit{et al.}~\cite{PengiGaneshWrightXu12} where they used the original RPCA formulation. Due to the limitations demonstrated previously, and time consuming convergence of RPCA-PCP, the algorithm is not suitable for real-time performance. Following the proposed optimization problem the individual images are regarded as frames of a video sequence concatenated in the data matrix $A$. The results will yield a decomposition in which $A \circ \tau$ describes the aligned faces in a canonical frame, the low-rank matrix $L$ that describes an eigen representation of a face of a person clear of corruptions and misalignment, and the matrix $S + G$ that contains the collective corruptions, specularities, shadows and noise. Figures~\ref{fig:Alignment},~\ref{fig:AverageBarack}, and~\ref{fig:LFWAll} demonstrate the results for a set of face images taken from the Labeled Faces in the Wild (LFW) dataset~\cite{LFW2007}.

\begin{figure}[!t]
\centering
\begin{subfigure}{.22\textwidth} 
  \centering
  \includegraphics[trim =10.3cm 6.1cm 9.9cm 6cm, clip = true,width=1\linewidth]{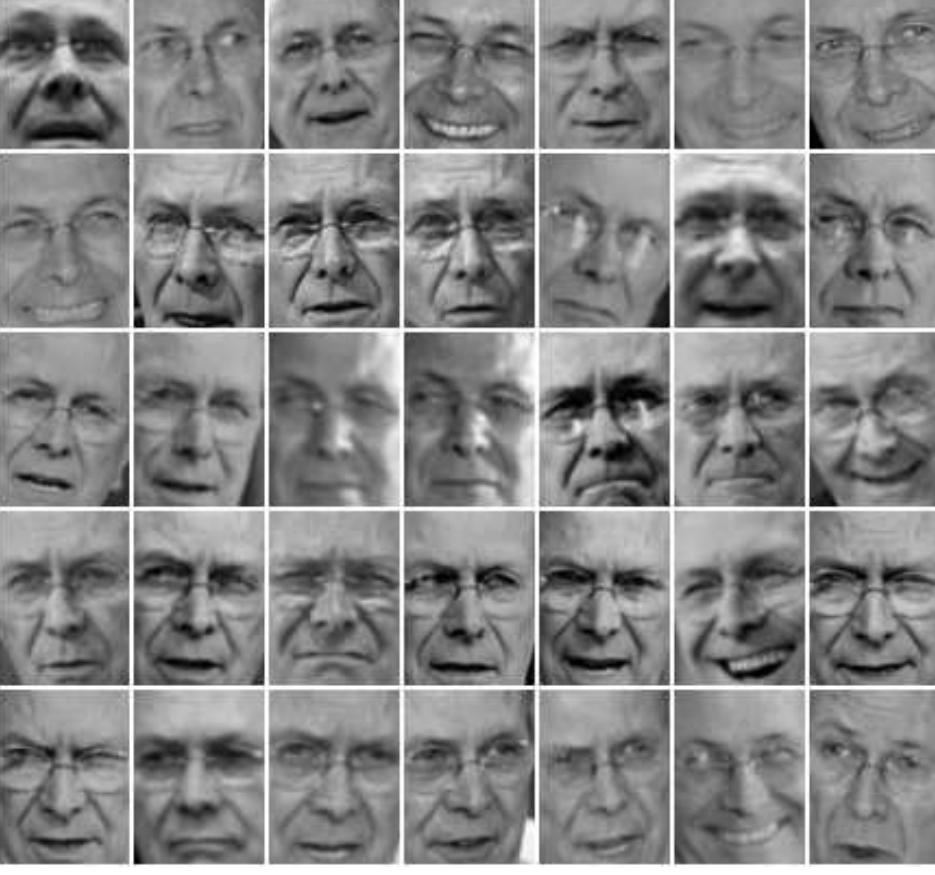}
  \caption{{\scriptsize Original images $A$}}
  \label{fig:ssfig1}
\end{subfigure}
\begin{subfigure}{.22\textwidth}
  \centering
  \includegraphics[trim =10.3cm 6.1cm 9.9cm 6cm, clip = true,width=1\linewidth]{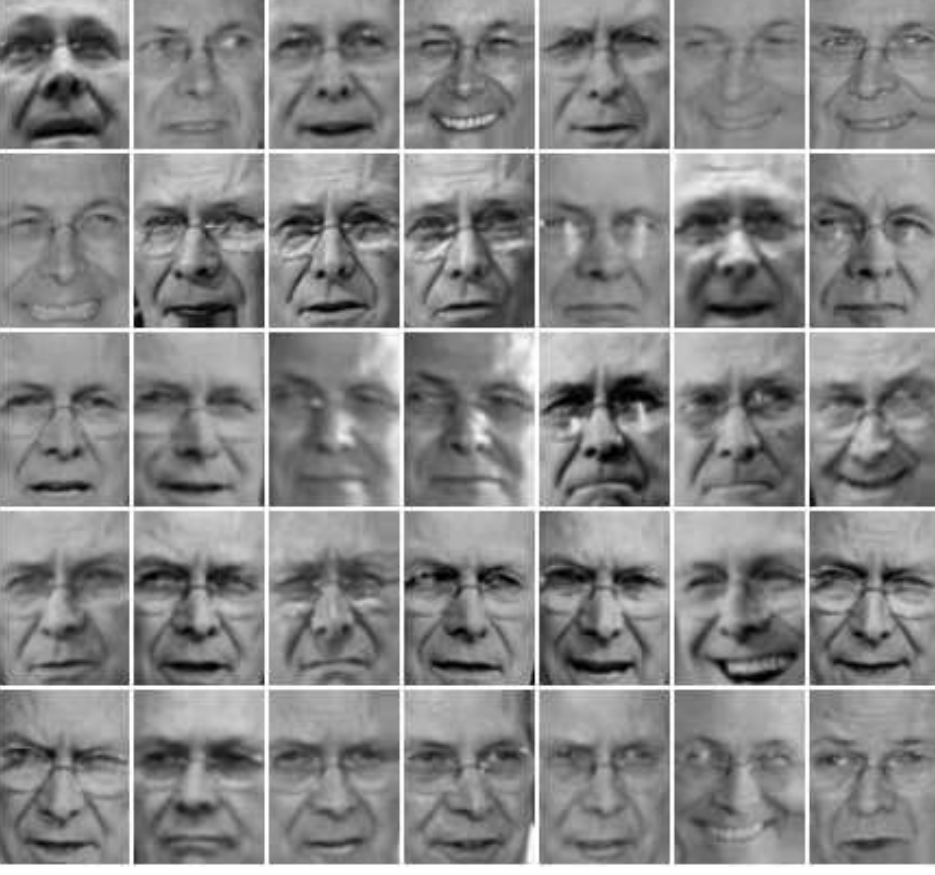}
  \caption{{\scriptsize Aligned images $A \circ \tau$}} 
  \label{fig:ssfig2}
\end{subfigure}
\begin{subfigure}{.22\textwidth}
  \centering
  \includegraphics[trim =10.3cm 6.1cm 9.9cm 6cm, clip = true,width=1\linewidth]{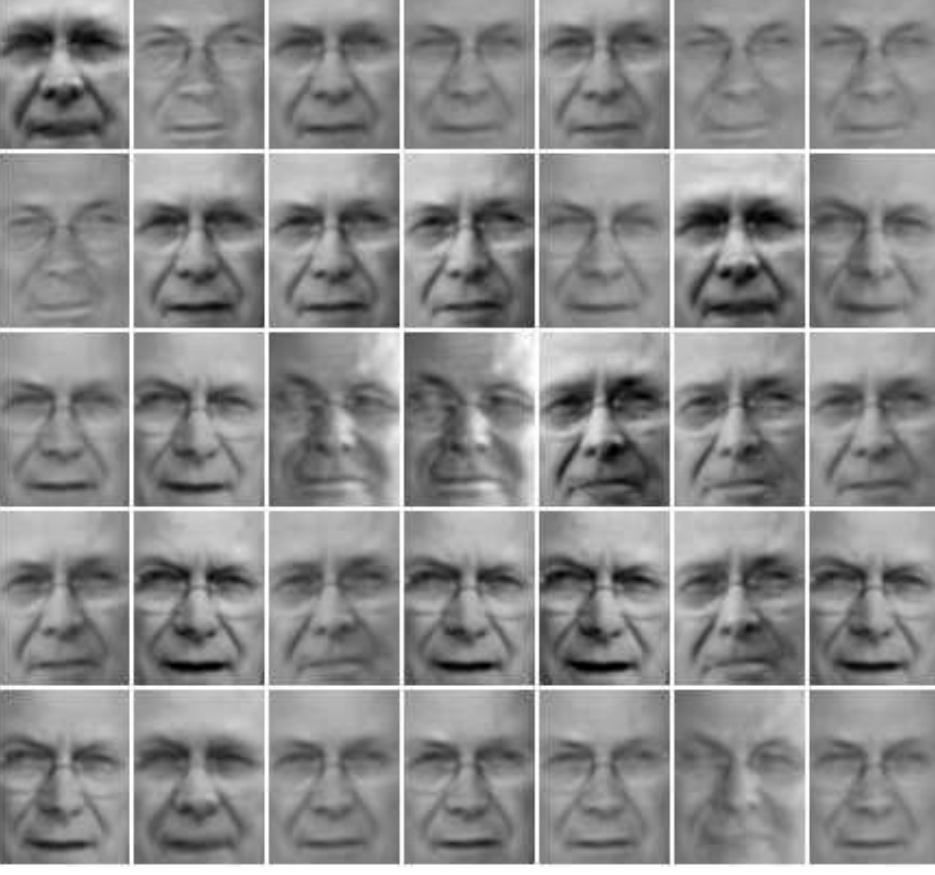}
  \caption{{\scriptsize Low-rank component $L$}}
  \label{fig:ssfig3}
\end{subfigure}
\begin{subfigure}{.22\textwidth}
  \centering
  \includegraphics[trim =10.3cm 6.1cm 9.9cm 6cm, clip = true,width=1\linewidth]{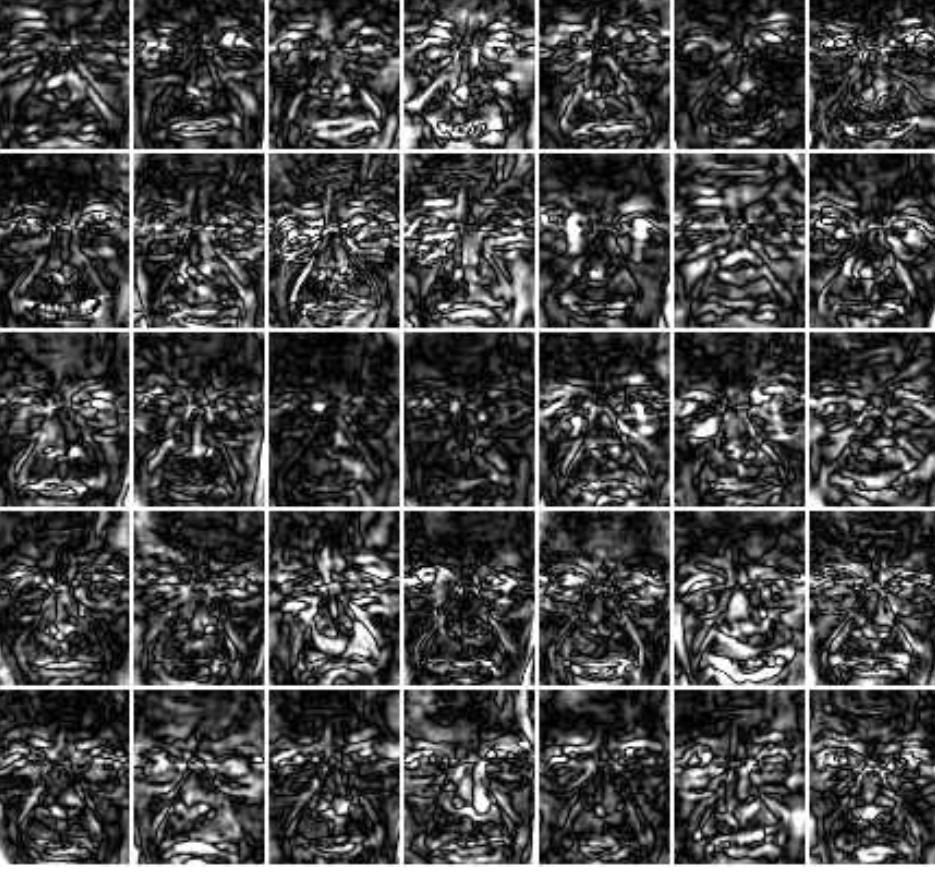} 
  \caption{{\scriptsize Sparse specularities $S + G$}}
  \label{fig:ssfig4}
\end{subfigure}
\begin{subfigure}{.1\textwidth} 
  \centering
  \includegraphics[trim =10.6cm 5cm 10.6cm 5cm, clip = true,width=1\linewidth]{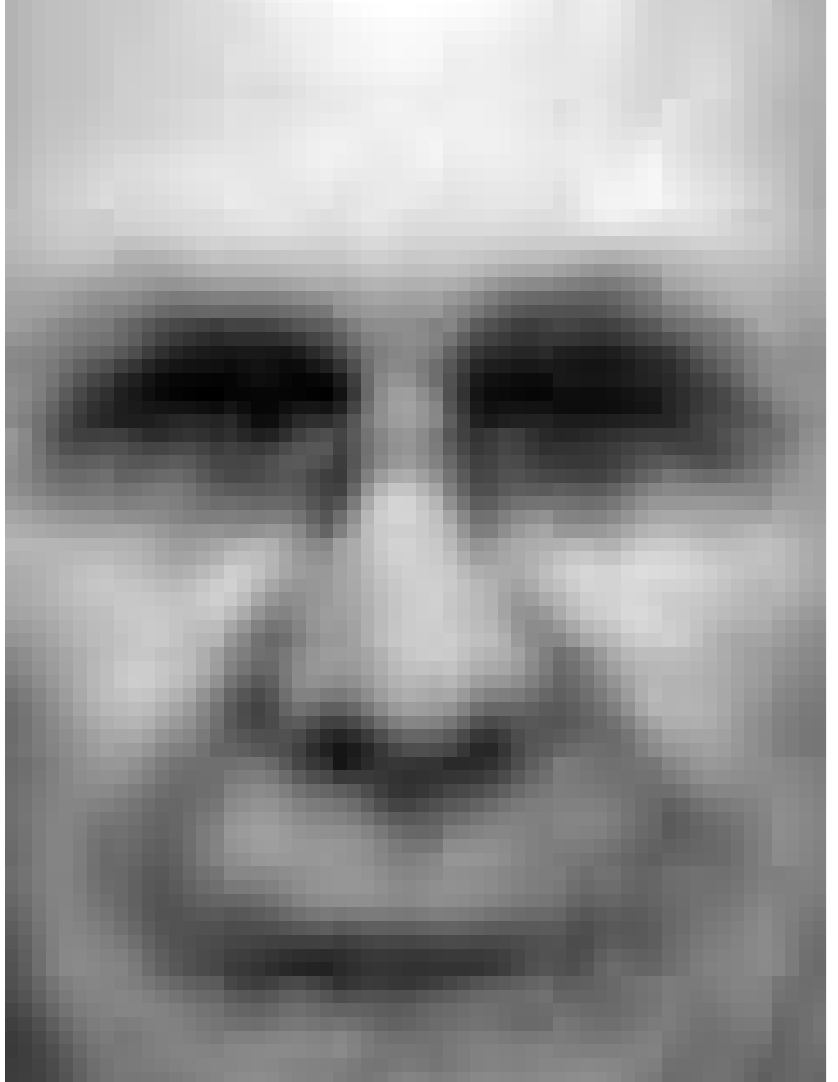}
  \caption{}
  \label{fig:sdfig1}
\end{subfigure}
\begin{subfigure}{.1\textwidth}
  \centering
  \includegraphics[trim =10.6cm 5cm 10.6cm 5cm, clip = true,width=1\linewidth]{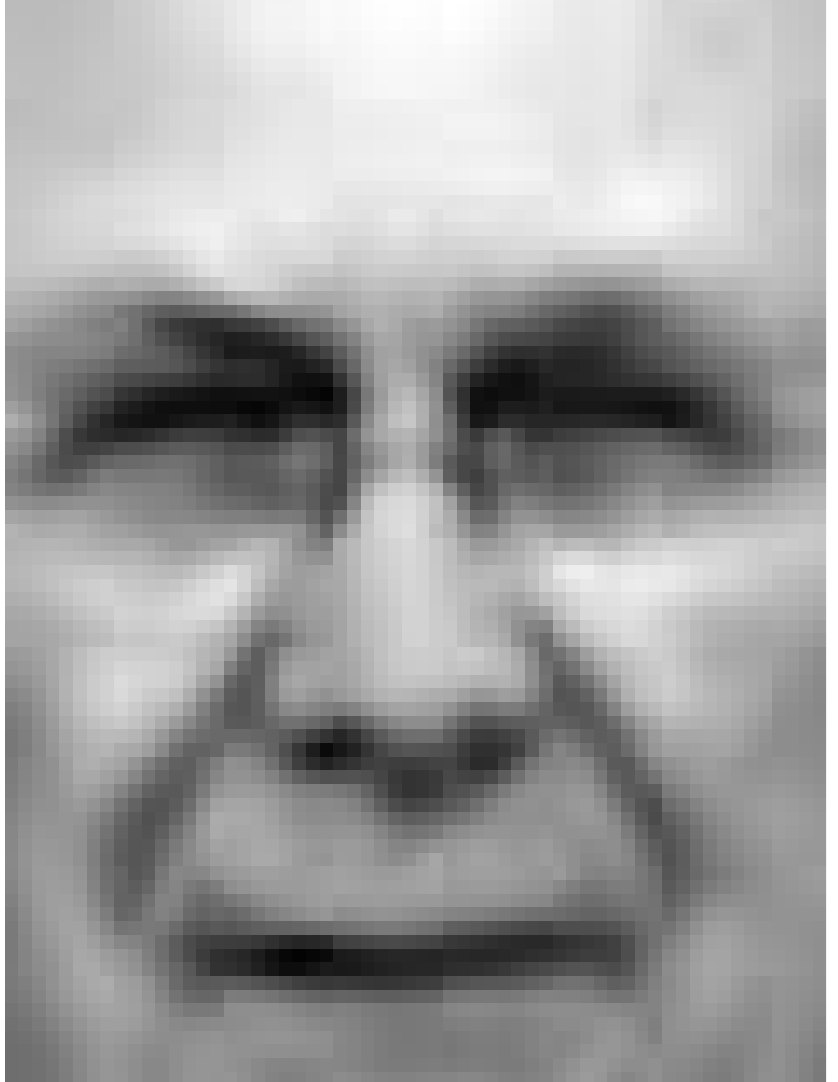}
  \caption{}
  \label{fig:sdfig1}
\end{subfigure}
\begin{subfigure}{.1\textwidth}
  \centering
  \includegraphics[trim =10.6cm 5cm 10.6cm 5cm, clip = true,width=1\linewidth]{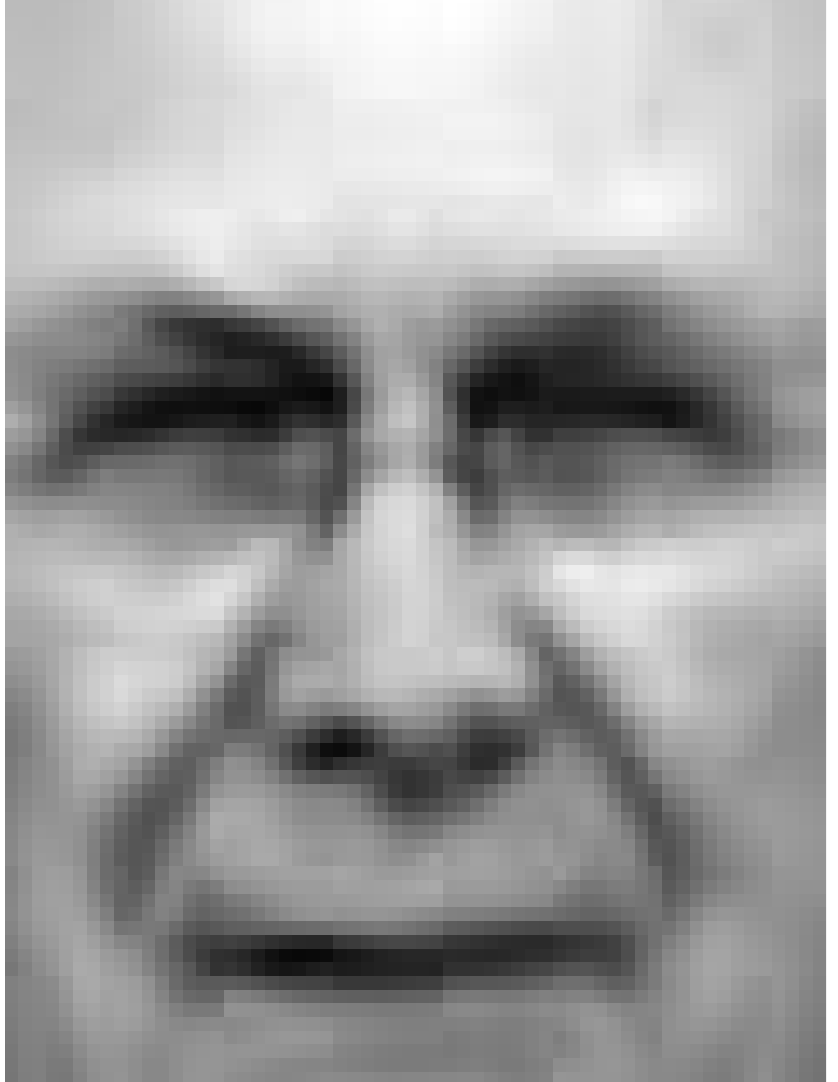}
  \caption{}
  \label{fig:sdfig1}
\end{subfigure}
\begin{subfigure}{.1\textwidth}
  \centering
  \includegraphics[trim =10.6cm 5cm 10.6cm 5cm, clip = true,width=1\linewidth]{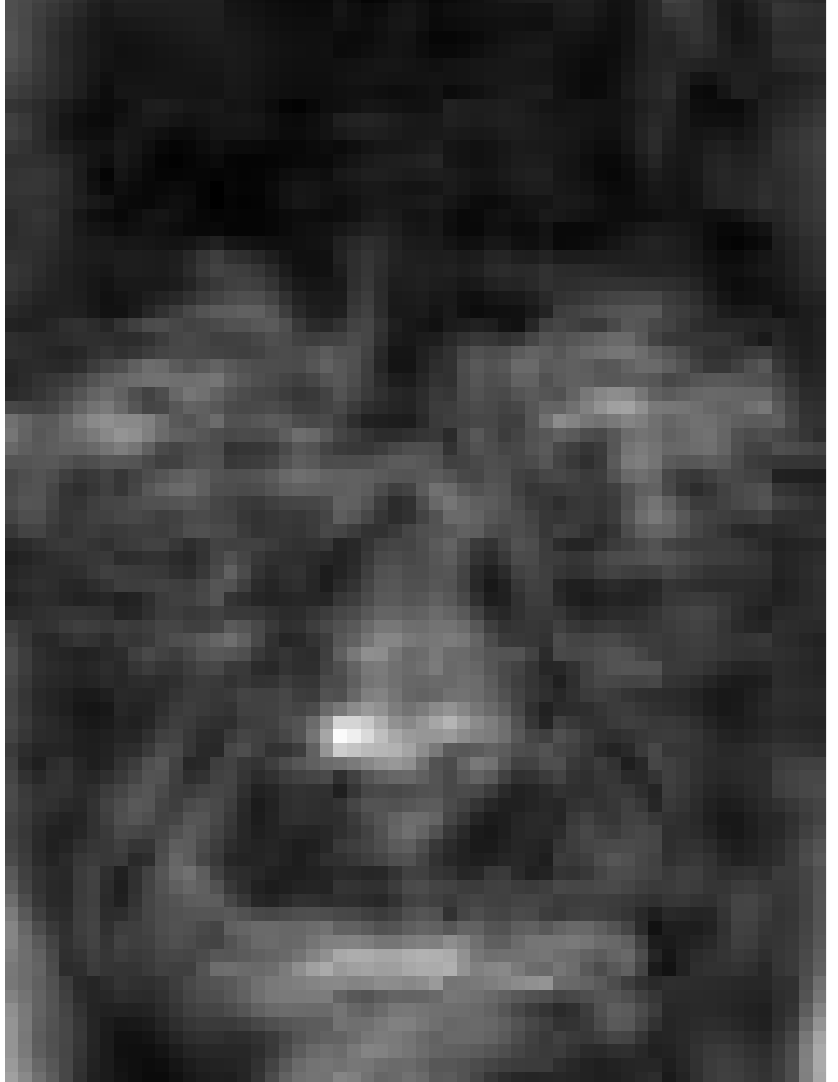}
  \caption{}
  \label{fig:sdfig1}
\end{subfigure}
\caption{Robust alignment by sparse and low-rank decomposition in LFW dataset~\cite{LFW2007}. Figures (e), (f), (g), and (h) correspond to the average of (a), (b), (c), and (d) respectively.}
\label{fig:Alignment}
\end{figure}
\begin{figure}[!t]
\centering
\begin{subfigure}{.1\textwidth} 
  \centering
  \includegraphics[trim =10.6cm 5cm 10.6cm 5cm, clip = true,width=1\linewidth]{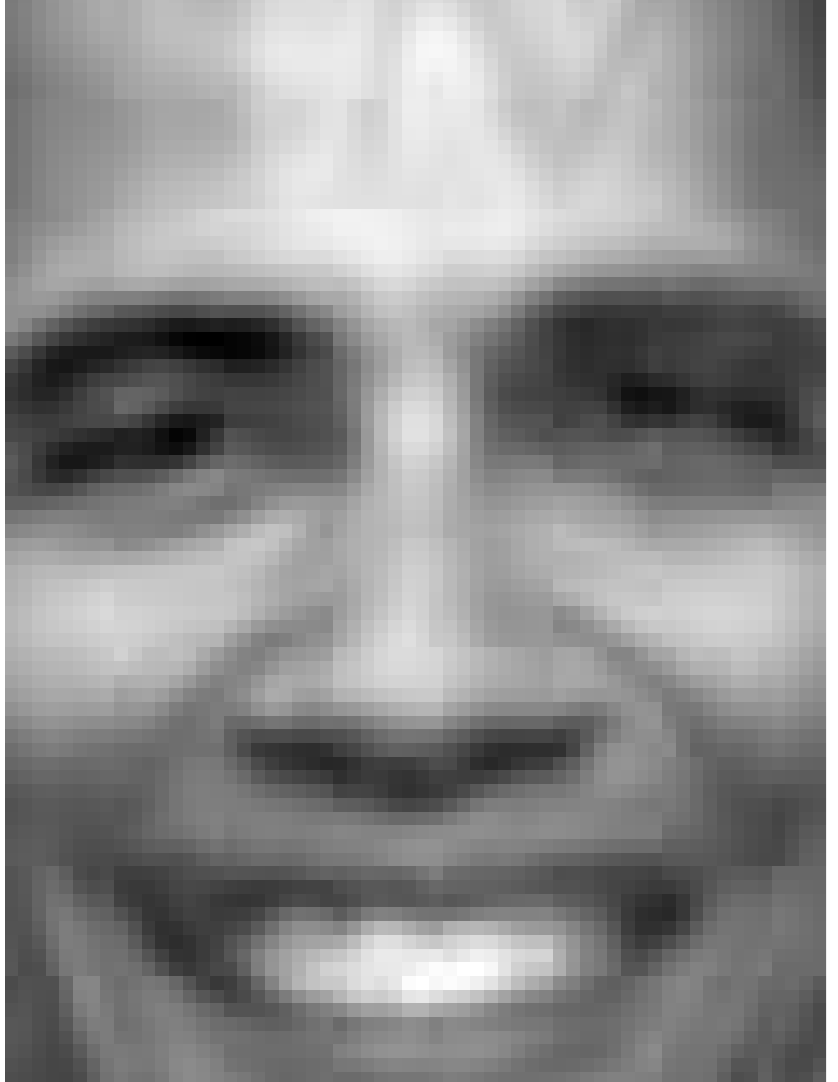}
  \caption{}
  \label{fig:sbfig1}
\end{subfigure}
\begin{subfigure}{.1\textwidth}
  \centering
  \includegraphics[trim =10.6cm 5cm 10.6cm 5cm, clip = true,width=1\linewidth]{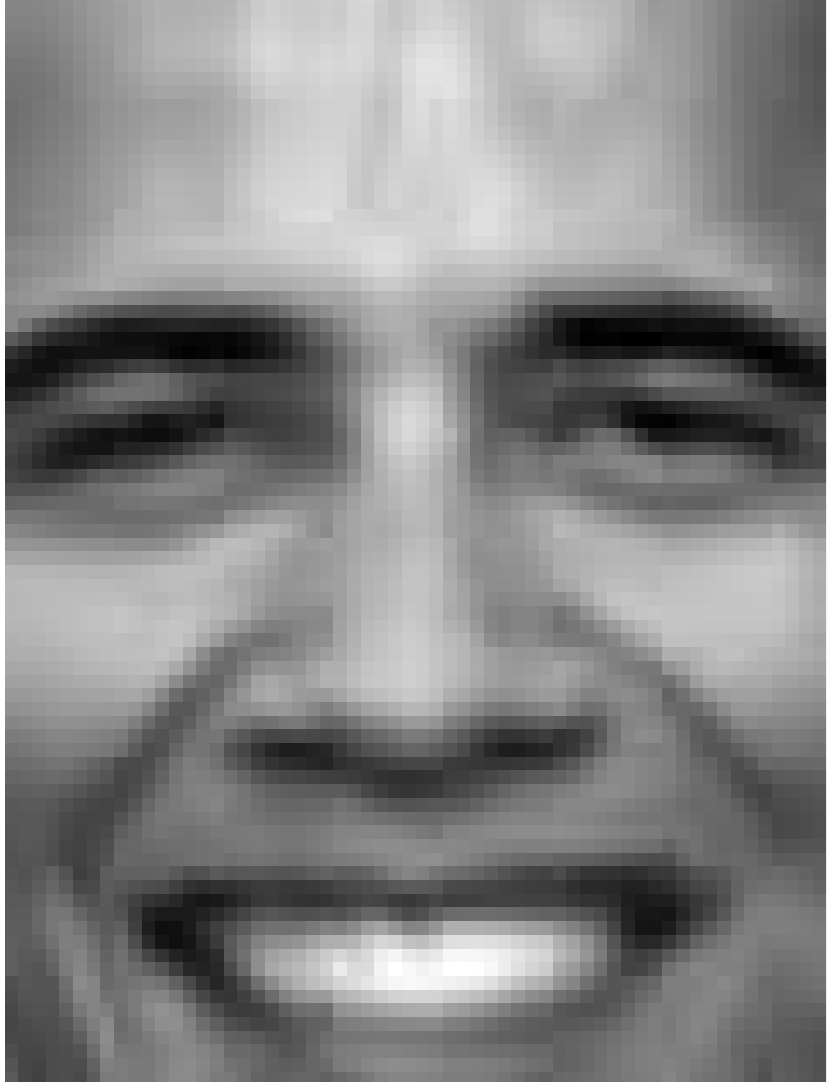}
  \caption{}
  \label{fig:sbfig1}
\end{subfigure}
\begin{subfigure}{.1\textwidth}
  \centering
  \includegraphics[trim =10.6cm 5cm 10.6cm 5cm, clip = true,width=1\linewidth]{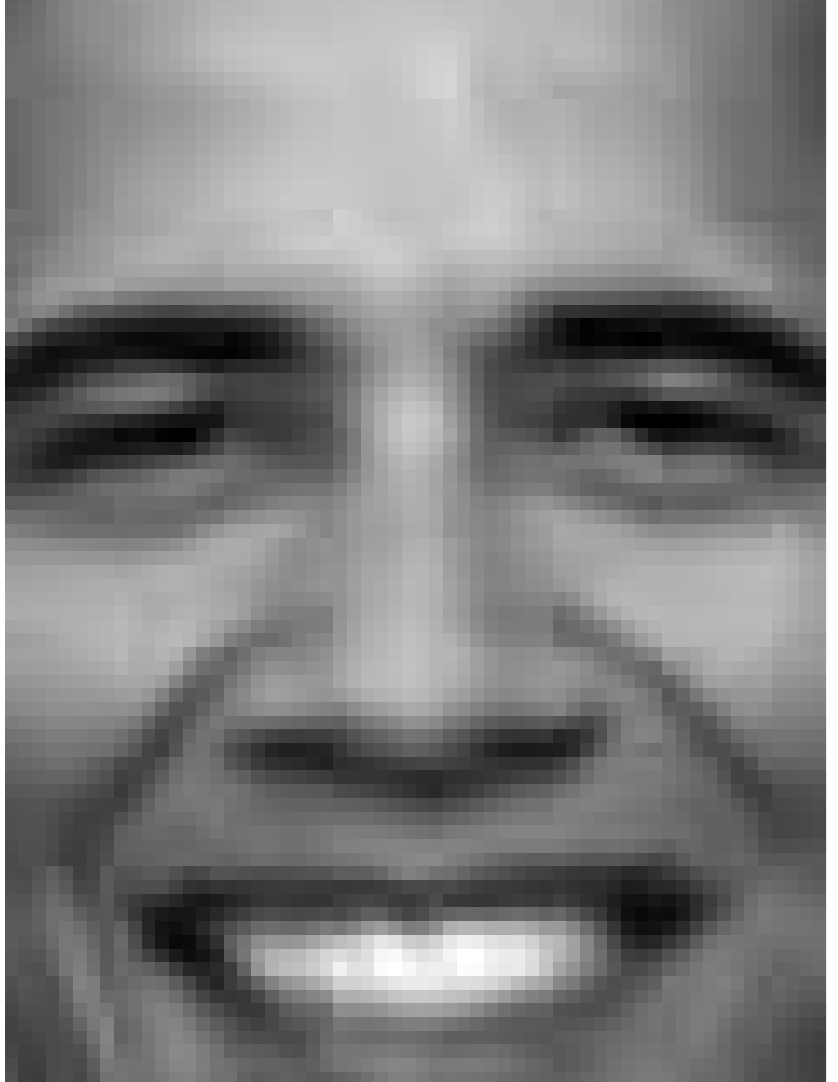}
  \caption{}
  \label{fig:sbfig1}
\end{subfigure}
\begin{subfigure}{.1\textwidth}
  \centering
  \includegraphics[trim =10.6cm 5cm 10.6cm 5cm, clip = true,width=1\linewidth]{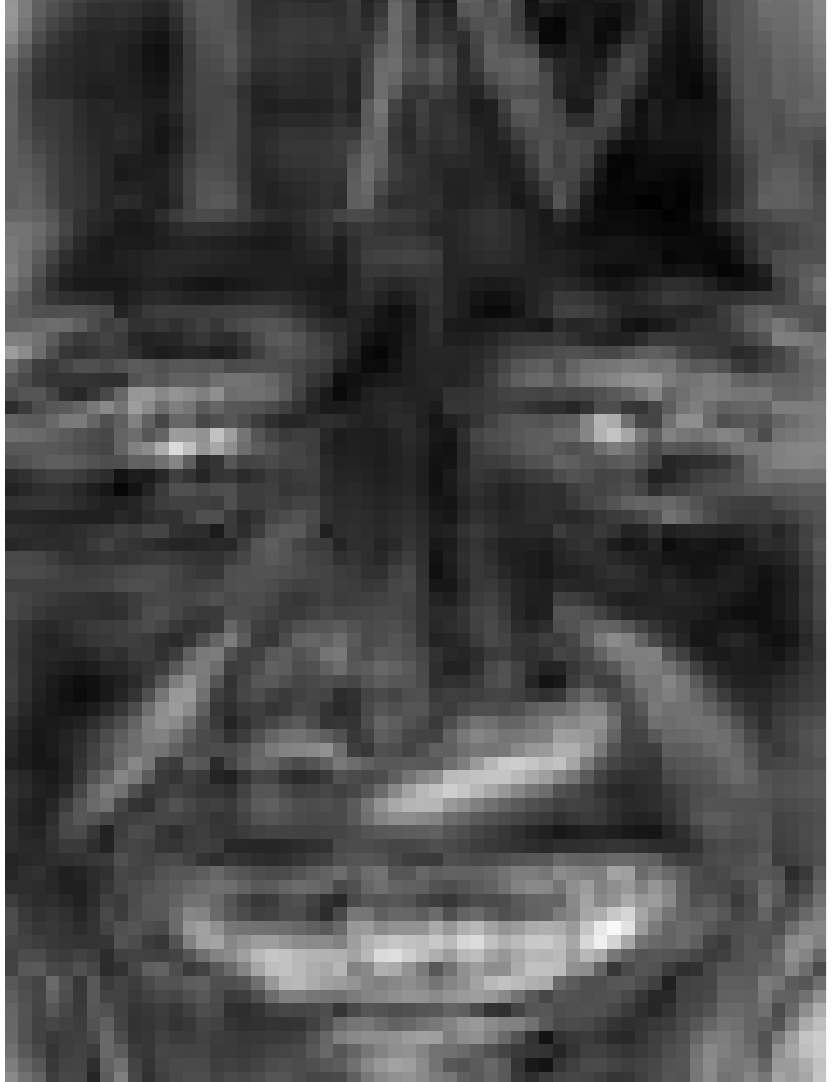}
  \caption{}
  \label{fig:sbfig1}
\end{subfigure}
\caption{Removing shadows and corruptions on faces from LFW dataset~\cite{LFW2007}. (a), (b), (c), and (d) correspond to average of: original images, aligned images, low-rank component, and sparse specularities respectively.}
\label{fig:AverageBarack}
\end{figure}
\begin{figure}[!t]
\centering
\begin{subfigure}{.23\textwidth} 
  \centering
  \includegraphics[trim =6.8cm 9.6cm 6.2cm 9.4cm, clip = true,width=1\linewidth]{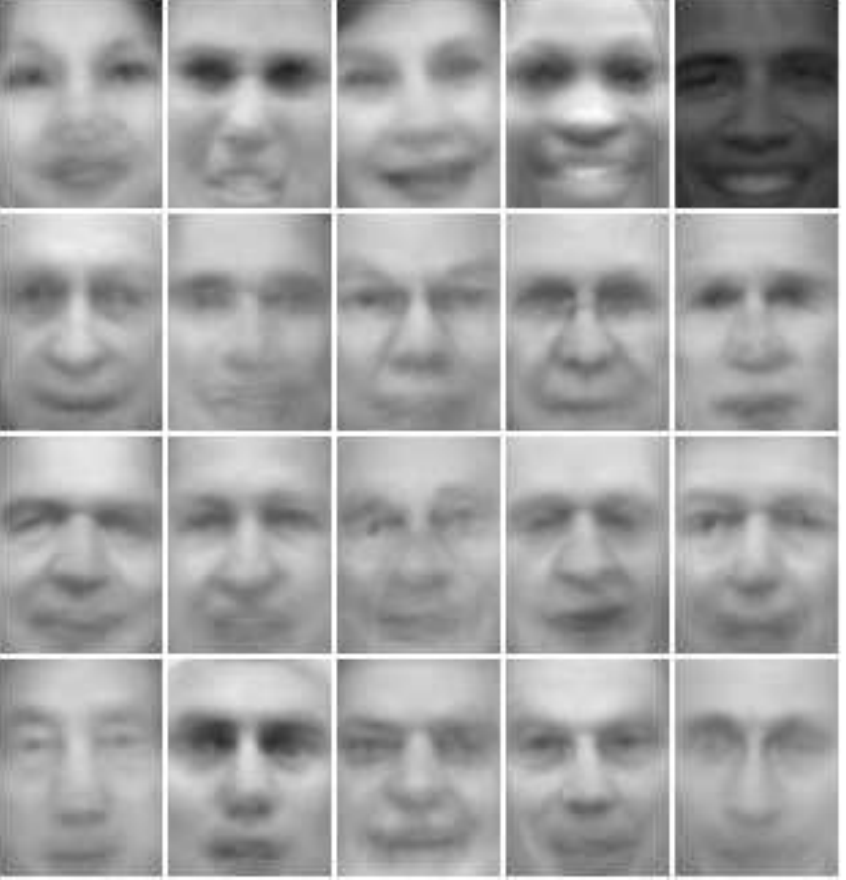}
  \caption{}
  \label{fig:sssfig1}
\end{subfigure}
\begin{subfigure}{.23\textwidth}
  \centering
  \includegraphics[trim =6.8cm 9.6cm 6.2cm 9.4cm, clip = true,width=1\linewidth]{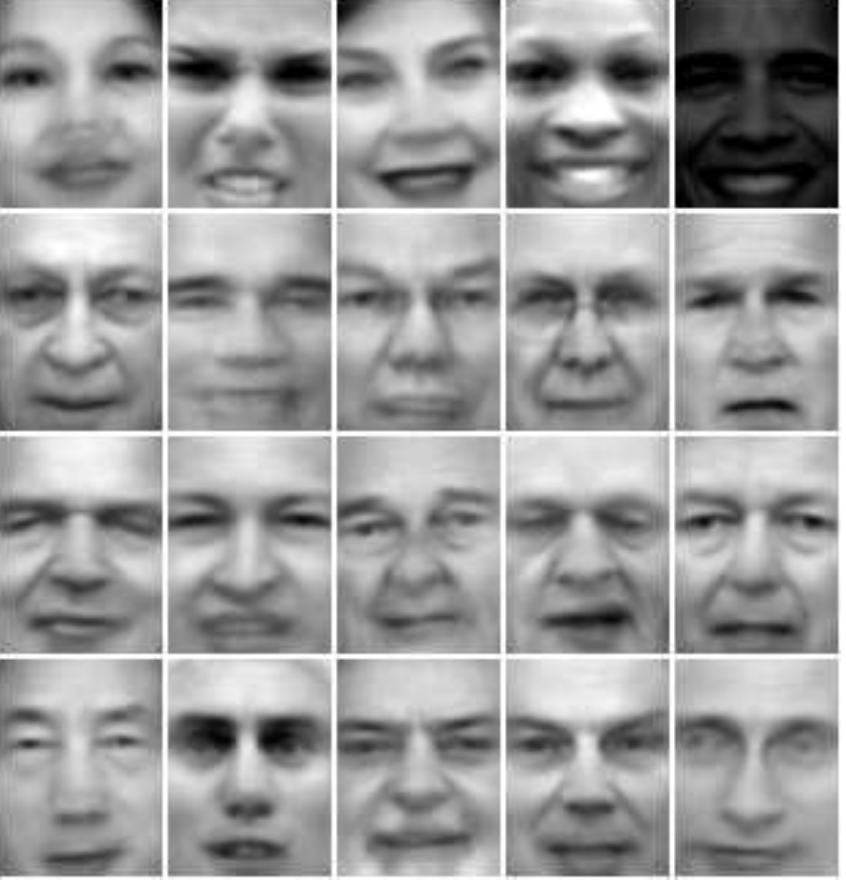}
  \caption{}
  \label{fig:sssfig2}
\end{subfigure}
\caption{Face alignment in large datasets. Average faces (a) before and (b) after alignment in LFW dataset~\cite{LFW2007} for 35 images per subject.}
\label{fig:LFWAll}
\end{figure}

%% file: IEEE_05_DiscussFuture.tex
In this article an improved approximated RPCA algorithm was proposed that aims to solve general scene background subtraction, as well as a novel SVD-free solution to the optimization problem for fast computation. The solutions presented in this paper aim to simultaneously solve some of the persistent issues that arise with RPCA-based methods in foreground/background segmentation of general video sequences. The proposed method can handle camera movement, various foreground object sizes, and slow-moving foreground pixels as well as sudden and gradual illumination changes in a scene. The qualitative and quantitative segmentation results outperform current state-of-the-art methods. The proposed SVD-free solution achieves more than double the amount of speed-up in computation time for the same performance target compared to its counterparts. In future the authors would like to move towards more unconstrained cases where the captured video could have any motion, parametric transformation, quality, motion blur, or deformation of scene elements. There is an ongoing work to compose a background library from the sequences solely for reconstruction purposes in video coding applications. In addition, the authors are currently studying the possibility of a multi-variable minimization problem based on the proposed original model, in order to incorporate more features for more robust detection and better reconstruction.